\documentclass{article} % For LaTeX2e
\usepackage{iclr2020_conference, times}

\usepackage[utf8]{inputenc} 
\usepackage[T1]{fontenc}    
\usepackage{amssymb}       
\usepackage{hyperref}       
\usepackage{url}            
\usepackage{booktabs}       
\usepackage{amsfonts}       
\usepackage{nicefrac}       
\usepackage{microtype}      
\usepackage{algorithm}
\usepackage{algorithmic}
\usepackage{amsmath}
\usepackage{times}
\usepackage{mathtools}
\usepackage{color}
\usepackage{siunitx}
\usepackage{multirow}
\usepackage{amsthm}
\usepackage{mathrsfs}
\usepackage{graphicx}
\usepackage{caption}
\usepackage{xspace}
\usepackage[export]{adjustbox}
\usepackage{float}
\usepackage{enumitem}
\usepackage{listings}
\usepackage{multirow}
\usepackage{bbm}
\usepackage{wrapfig}
\usepackage{amsthm}
\usepackage{subcaption}
\usepackage{wrapfig}
\usepackage{placeins}

\usepackage{amsmath,amsfonts,bm}

\def\1{\bm{1}}

\DeclareMathAlphabet{\mathsfit}{\encodingdefault}{\sfdefault}{m}{sl}
\SetMathAlphabet{\mathsfit}{bold}{\encodingdefault}{\sfdefault}{bx}{n}

\newtheorem{proposition}{Proposition}

\newcommand{\mine}{I_\texttt{NWJ}}
\newcommand{\infonce}{I_\texttt{NCE}}
\newcommand{\iest}{I_\texttt{EST}}
\newcommand{\inwj}{I_\texttt{NWJ}}

\usepackage{xcolor}
\hypersetup{
    colorlinks,
    anchorcolor={blue!50!black},
    linkcolor={blue!50!black},
    citecolor={blue!50!black},
    urlcolor={blue}
}

\allowdisplaybreaks

\title{On Mutual Information Maximization for Representation Learning}

\author{Michael Tschannen\thanks{Equal contribution. Correspondence to Michael Tschannen (tschannen@google.com), Josip Djolonga (josipd@google.com), and Mario Lucic (lucic@google.com). $^\dagger$PhD student at University of Cambridge and the Max Planck Institute for Intelligent Systems, Tübingen.
}
\quad
Josip Djolonga\footnotemark[1]
\quad
Paul K. Rubenstein$^\dagger$
\quad
Sylvain Gelly
\quad
Mario Lucic
\\
\normalfont{Google Research, Brain Team}
}
\iclrfinalcopy

\begin{document}
% {\centering
\maketitle
% }
\vspace{-2mm}
\begin{abstract}
Many recent methods for unsupervised or self-supervised representation learning train feature extractors by maximizing an estimate of the mutual information (MI) between different views of the data.
This comes with several immediate problems: For example, MI is notoriously hard to estimate,
and using it as an objective for representation learning may lead to highly entangled representations due to its invariance under arbitrary invertible transformations.
Nevertheless, these methods have been repeatedly shown to excel in practice.
In this paper we argue, and provide empirical evidence, that the success of these methods cannot be attributed to the properties of MI alone, and that they strongly depend on the inductive bias in both the choice of feature extractor architectures and the parametrization of the employed MI estimators. Finally, we establish a connection to deep metric learning and argue that this interpretation may be a plausible explanation for the success of the recently introduced methods.

\end{abstract}

\section{Introduction}

Unsupervised representation learning is a fundamental problem in machine learning.
Intuitively, one aims to learn a function $g$ which maps the data into some, usually lower-dimensional, space where one can solve some (generally a priori unknown) target supervised tasks more efficiently, i.e.\ with fewer labels.
In contrast to supervised and semi-supervised learning, the learner has access \emph{only to unlabeled data}.
Even though the task seems ill-posed as there is no natural objective one should optimize, by leveraging domain knowledge this approach can be successfully applied to a variety of problem areas, including image~\citep{kolesnikov2019revisiting, oord2018representation, henaff2019data, tian2019contrastive, hjelm2018learning, bachman2019learning} and video classification~\citep{wang2015unsupervised, sun2019contrastive}, and natural language understanding~\citep{oord2018representation, peters2018deep, devlin2018bert}.

Recently, there has been a revival of approaches inspired by the \emph{InfoMax principle}~\citep{linsker1988self}: Choose a representation $g(x)$ maximizing the mutual information (MI) between the input and its representation, possibly subject to some structural constraints.
MI measures the amount of information obtained about a random variable $X$ by observing some other random variable $Y$\footnote{We denote random variables using upper-case letters (e.g. $X$, $Y$), and their realizations by the corresponding lower-case letter (e.g. $x$, $y$).}
Formally, the MI between $X$ and $Y$, with joint density $p(x,y)$ and marginal densities $p(x)$ and $p(y)$, is defined as the Kullback–Leibler (KL) divergence between the joint and the product of the marginals
\begin{align}\label{eq:kldiv}
I(X ; Y)= D_{\textrm{KL}}\left(p(x, y) \,\|\, p(x)p(y) \right) = \mathbb{E}_{p(x, y)} \left[\log \frac{p(x, y)}{p(x)p(y)} \right].
\end{align}
The fundamental properties of MI are well understood and have been extensively studied (see e.g.\ \citet{kraskov2004estimating}).
Firstly, MI is invariant under reparametrization of the variables --- namely, if $X' = f_1(X)$ and $Y' = f_2(Y)$ are homeomorphisms (i.e.\ smooth invertible maps), then $I(X; Y) = I(X'; Y')$.
Secondly, estimating MI in high-dimensional spaces is a notoriously difficult task, and in practice one often maximizes a tractable lower bound on this quantity~\citep{pmlr-v97-poole19a}.
Nonetheless, any distribution-free high-confidence lower bound on entropy requires a sample
size exponential in the size of the bound~\citep{mcallester2018formal}. 

Despite these fundamental challenges, several recent works have demonstrated promising empirical results in representation learning using MI maximization~\citep{oord2018representation, henaff2019data, tian2019contrastive, hjelm2018learning, bachman2019learning, sun2019contrastive}. In this work we argue, and provide empirical evidence, that the success of these methods cannot be attributed to the properties of MI alone.
In fact, we show that maximizing tighter bounds on MI can result in worse representations.
In addition, we establish a connection to deep metric learning and argue that this interpretation may be a plausible explanation of the success of the recently introduced methods.\footnote{The code for running the experiments and visualizing the results is available at \href{https://github.com/google-research/google-research/tree/master/mutual_information_representation_learning}{https://github.com/google-research/google-research/tree/master/mutual\_information\_representation\_learning}.} 
%\footnote{The code for running the experiments and visualizing the results is available at \href{https://github.com/google-research/google-research/tree/master/mutual_information_representation_learning}{https://github.com/google-research/google-research/tree/master/mutual\_information\_representation\_learning}.}

\section{Background and Related Work} 

\textbf{Recent progress and the InfoMax principle}\quad While promising results in other domains have been presented in the literature, we will focus on unsupervised image representation learning techniques that have achieved state-of-the-art performance on image classification tasks~\citep{henaff2019data, tian2019contrastive, bachman2019learning}.
The usual problem setup dates back at least to \citet{becker1992self} and can conceptually be described as follows: For a given image $X$, let $X^{(1)}$ and $X^{(2)}$ be different, possibly overlapping \emph{views} of $X$, for instance the top and bottom halves of the image.
These are encoded using encoders $g_1$ and $g_2$ respectively, and the MI between the two representations $g_1(X^{(1)})$ and $g_2(X^{(2)})$ is maximized, 
\begin{equation}
\max_{g_1 \in \mathcal{G}_1, g_2 \in \mathcal{G}_2} \quad \iest\left(g_1(X^{(1)}); g_2(X^{(2)})\right), \label{eq:newinfomax}
\end{equation}
where $\iest (X ; Y)$ is a sample-based estimator of the true MI $I(X ; Y)$ and the function classes $\mathcal{G}_1$ and $\mathcal{G}_2$ can be used to specify structural constraints on the encoders. While not explicitly reflected in~\eqref{eq:newinfomax}, note that $g_1$ and $g_2$ can often share parameters.
Furthermore, it can be shown that $I(g_1(X^{(1)}); g_2(X^{(2)})) \leq I(X; g_1(X^{(1)}), g_2(X^{(2)}))$,\footnote{Follows from the data processing inequality (see Prop.\ \ref{prop:newinfomax} in Appendix~\ref{app:newinfomax}). 
} hence the objective in \eqref{eq:newinfomax} can be seen as a lower bound on the InfoMax objective $\max_{g \in \mathcal{G}} I(X; g(X))$ \citep{linsker1988self}.

\textbf{Practical advantages of multi-view formulations}\quad
There are two main advantages in using \eqref{eq:newinfomax} rather than the original InfoMax objective.
First, the MI has to be estimated only between the learned representations of the two views, which typically lie on a much lower-dimensional space than the one where the original data $X$ lives.
Second, it gives us plenty of modeling flexibility, as the two views can be chosen to capture completely different aspects and modalities of the data, for example:
\begin{enumerate}[itemsep=0pt,topsep=0pt, leftmargin=20pt]
\item In the basic form of \emph{DeepInfoMax}~\citep{hjelm2018learning} $g_1$ extracts global features from the entire image $X^{(1)}$ and $g_2$ local features from image patches $X^{(2)}$, where $g_1$ and $g_2$ correspond to activations in different layers of the same convolutional network. \citet{bachman2019learning} build on this and compute the two views from different augmentations of the same image.
\item \emph{Contrastive multiview coding} (CMC)~\citep{tian2019contrastive} generalizes the objective in \eqref{eq:newinfomax} to consider multiple views $X^{(i)}$, where each $X^{(i)}$ corresponds to a different image modality (e.g., different color channels, or the image and its segmentation mask). 
\item \emph{Contrastive predictive coding} (CPC) \citep{oord2018representation, henaff2019data} incorporates a sequential component of the data. Concretely, one extracts a sequence of patches from an image in some fixed order, maps each patch using an encoder, aggregates the resulting features of the first $t$ patches into a context vector, and maximizes the MI between the context and features extracted from the patch at position $t+k$. In \eqref{eq:newinfomax}, $X^{(1)}$ would thus correspond to the first $t$ patches and $X^{(2)}$ to the patch at location $t+k$.
\end{enumerate}
Other approaches, such as those presented by \citet{sermanet2018time}, \citet{hu2017learning}, and \citet{ji2018invariant}, can be similarly subsumed under the same objective.

\textbf{Lower bounds on MI}\label{sec:MILB}\quad 
As evident from \eqref{eq:newinfomax}, another critical choice is the MI estimator $\iest$.
Given the fundamental limitations of MI estimation~\citep{mcallester2018formal}, recent work has focused on deriving lower bounds on MI~\citep{barber2003algorithm,belghazi2018mine,pmlr-v97-poole19a}.
Intuitively, these bounds are based on the following idea: If a classifier can accurately distinguish between samples drawn from the joint $p(x,y)$ and those drawn from the product of marginals $p(x)p(y)$, then $X$ and $Y$ have a high MI.

We will focus on two such estimators, which are most commonly used in the representation learning literature.
The first of them, termed \emph{InfoNCE}~\citep{oord2018representation}, is defined as
\begin{align}\label{eqn:infonce}
    I(X ; Y) \geq \mathbb{E}\left[\frac{1}{K} \sum_{i=1}^{K} \log \frac{e^{f\left(x_{i}, y_{i}\right)}}{\frac{1}{K} \sum_{j=1}^{K} e^{f\left(x_{i}, y_{j}\right)}}\right] \triangleq \infonce(X;Y),
\end{align}
where the expectation is over $K$ independent samples $\{(x_i, y_i)\}_{i=1}^K$ from the joint distribution $p(x, y)$~ \citep{pmlr-v97-poole19a}. 
In practice we estimate \eqref{eqn:infonce} using Monte Carlo estimation by averaging over multiple batches of samples.
Intuitively, the \emph{critic} function $f$ tries to predict for each $x_i$ which of the $K$ samples $y_1, \ldots, y_k$ it was jointly drawn with, by assigning high values to the jointly drawn pair, and low values to all other pairs.
The second estimator is based on the variational form of the KL divergence due to \emph{Nguyen, Wainwright, and Jordan (NWJ)} \citep{nguyen2010estimating} and takes the form
\begin{align}\label{eqn:iwj}
    I(X ; Y) \geq \mathbb{E}_{p(x, y)}[f(x, y)]-e^{-1} \mathbb{E}_{p(x)}[\mathbb{E}_{p(y)} e^{f(x,y)}] \triangleq \inwj(X;Y).
\end{align}
For detailed derivations we refer the reader to \citep{ruderman2012tighter,pmlr-v97-poole19a}.
Note that these bounds hold for any critic $f$ and when used in \eqref{eq:newinfomax} one in practice jointly maximizes over $g_1, g_2$ and $f$. 
\looseness-1 Furthermore, it can be shown that \eqref{eqn:infonce} is maximized by $f^*(x,y)=\log p(y|x)$ and \eqref{eqn:iwj} by $f^*(x,y)= 1 + \log p(y|x)$ \citep{pmlr-v97-poole19a}.
Common choices for $f$ include bilinear critics $f(x,y) = x^\top W y$ \citep{oord2018representation, henaff2019data, tian2019contrastive}, separable critics $f(x,y) = \phi_1(x)^\top \phi_2(y)$ \citep{bachman2019learning}, and concatenated critics $f(x,y) = \phi([x,y])$ \citep{hjelm2018learning} (here $\phi, \phi_1, \phi_2$ are typically shallow multi-layer perceptrons (MLPs)). When applying these estimators to solve \eqref{eq:newinfomax}, the line between the critic and the encoders $g_1, g_2$ can be blurry. For example, one can train with an inner product critic $f(x,y) = x^\top y$, but extract features from an intermediate layer of $g_1, g_2$, in which case the top layers of $g_1,g_2$ form a separable critic. Nevertheless, this boundary is crucial for the interplay between MI estimation and the interpretation of the learned representations.

%%%%%%%%%%%%%%%%%%%%%%%%%%%%%%%%%%%%%%%%%%%%%%%%%%%%%%%%%%%%%%%%%%%%%%%%%%%%%%%%%%%%%%%%%%%%%%%%%%%%
%%%%%%%%%%%%%%%%%%%%%%%%%%%%%%%%%%%%%%%%%%%%%%%%%%%%%%%%%%%%%%%%%%%%%%%%%%%%%%%%%%%%%%%%%%%%%%%%%%%%
\section{Biases in approximate information maximization}
%%%%%%%%%%%%%%%%%%%%%%%%%%%%%%%%%%%%%%%%%%%%%%%%%%%%%%%%%%%%%%%%%%%%%%%%%%%%%%%%%%%%%%%%%%%%%%%%%%%%
%%%%%%%%%%%%%%%%%%%%%%%%%%%%%%%%%%%%%%%%%%%%%%%%%%%%%%%%%%%%%%%%%%%%%%%%%%%%%%%%%%%%%%%%%%%%%%%%%%%%
It is folklore knowledge that maximizing MI does not necessarily lead to useful representations.
Already \cite{linsker1988self} talks in his seminal work about constraints, while a manifestation of the problem in clustering approaches using MI criteria has been brought up by \cite{bridle1992unsupervised} and subsequently addressed using regularization by \cite{krause2010discriminative}.
To what can we then attribute the recent success of methods building on the principles of MI maximization?
We will argue that their connection to the \emph{InfoMax} principle might be very loose.
Namely, we will show that they behave counter-intuitively if one equates them with MI maximization, and that the performance of these methods depends strongly on the bias that is encoded not only in the encoders, but also on the actual form of the used estimators.

\begin{enumerate}[itemsep=0pt,topsep=0pt, leftmargin=20pt]
\item We first consider encoders which are \emph{bijective} by design. Even though the true MI is maximized for any choice of model parameters, the representation quality (measured by downstream linear classification accuracy) improves during training. Furthermore, there exist invertible encoders for which the representation quality is \emph{worse} than using raw pixels, despite also maximizing MI.
\item We next consider encoders that can model both invertible and non-invertible functions. When the encoder can be non-invertible, but is initialized to be invertible, $\iest$ still biases the encoders to be very ill-conditioned and hard to invert.
\item For $\infonce$ and $\inwj$, higher-capacity critics admit tighter bounds on MI. We demonstrate that simple critics yielding loose bounds can lead to better representations than high-capacity critics.
\item Finally, we optimize the estimators to the same MI lower-bound value with different encoder architectures and show that the representation quality can be impacted more by the choice of the architecture, than the estimator. 
\end{enumerate}
As a consequence, we argue that the success of these methods and the way they are instantiated in practice is only loosely connected to MI. Then, in Section~\ref{sec:triplet} we provide an alternative explanation for the success of recent methods through a connection to classic triplet losses from metric learning.

\textbf{Setup}\quad 
Our goal is to provide a minimal set of easily reproducible empirical experiments to understand the role of MI estimators, critic and encoder architectures when learning representations via the objective \eqref{eq:newinfomax}.  
To this end, we consider a simple setup of learning a representation of the top half of MNIST handwritten digit images (we present results for the experiments from Sections~\ref{sec:criticarch} and ~\ref{sec:encoderarch} on CIFAR10 in Appendix~\ref{app:cifar10results}; the conclusions are analogous).
This setup has been used in the context of deep canonical correlation analysis \citep{andrew2013deep}, where the target is to maximize the correlation between the representations.
Following the widely adopted \emph{downstream linear evaluation protocol} \citep{kolesnikov2019revisiting, oord2018representation, henaff2019data, tian2019contrastive, hjelm2018learning, bachman2019learning}, we train a linear classifier\footnote{Using SAGA \citep{defazio2014saga}, as implemented in \texttt{scikit-learn} \citep{scikit-learn}.} for digit classification on the learned representation using all available training labels (other evaluation protocols are discussed in Section~\ref{sec:conclusion}).
To learn the representation we instantiate \eqref{eq:newinfomax} and split each input MNIST image $x \in [0, 1]^{784}$ into two parts, the top part of the image $x_\text{top} \in [0, 1]^{392}$ corresponding to $X^{(1)}$, and the bottom part, $x_\text{bottom} \in [0, 1]^{392}$, corresponding to $X^{(2)}$, respectively.
We train $g_1$, $g_2$, and $f$ using the Adam optimizer \citep{kingma2014adam}, and use $g_1(x_\text{top})$ as the representation for the linear evaluation.
Unless stated otherwise, we use a bilinear critic $f(x,y) = x^\top W y$ (we investigate its effect in a separate ablation study), set the batch size to $128$ and the learning rate to $10^{-4}$.\footnote{Note that $\infonce$ is upper-bounded by $\log(\text{batch size}) \approx 4.85$ \citep{oord2018representation}. We experimented with batch sizes up to $512$ and obtained consistent results aligned with the stated conclusions.} Throughout, $\iest$ values and downstream classification accuracies are  averaged over $20$ runs and reported on the testing set (we did not observe large gaps between the training and testing values of $\iest$). As a common baseline, we rely on a linear classifier in pixel space on $x_\text{top}$, which obtains a testing accuracy of about $85\%$. For comparison, a simple MLP or ConvNet architecture achieves about $94\%$ (see Section~\ref{sec:encoderarch} for details).

\begin{figure}[t]
    \centering
    \begin{subfigure}{0.32\textwidth}
    \includegraphics[width=\textwidth]{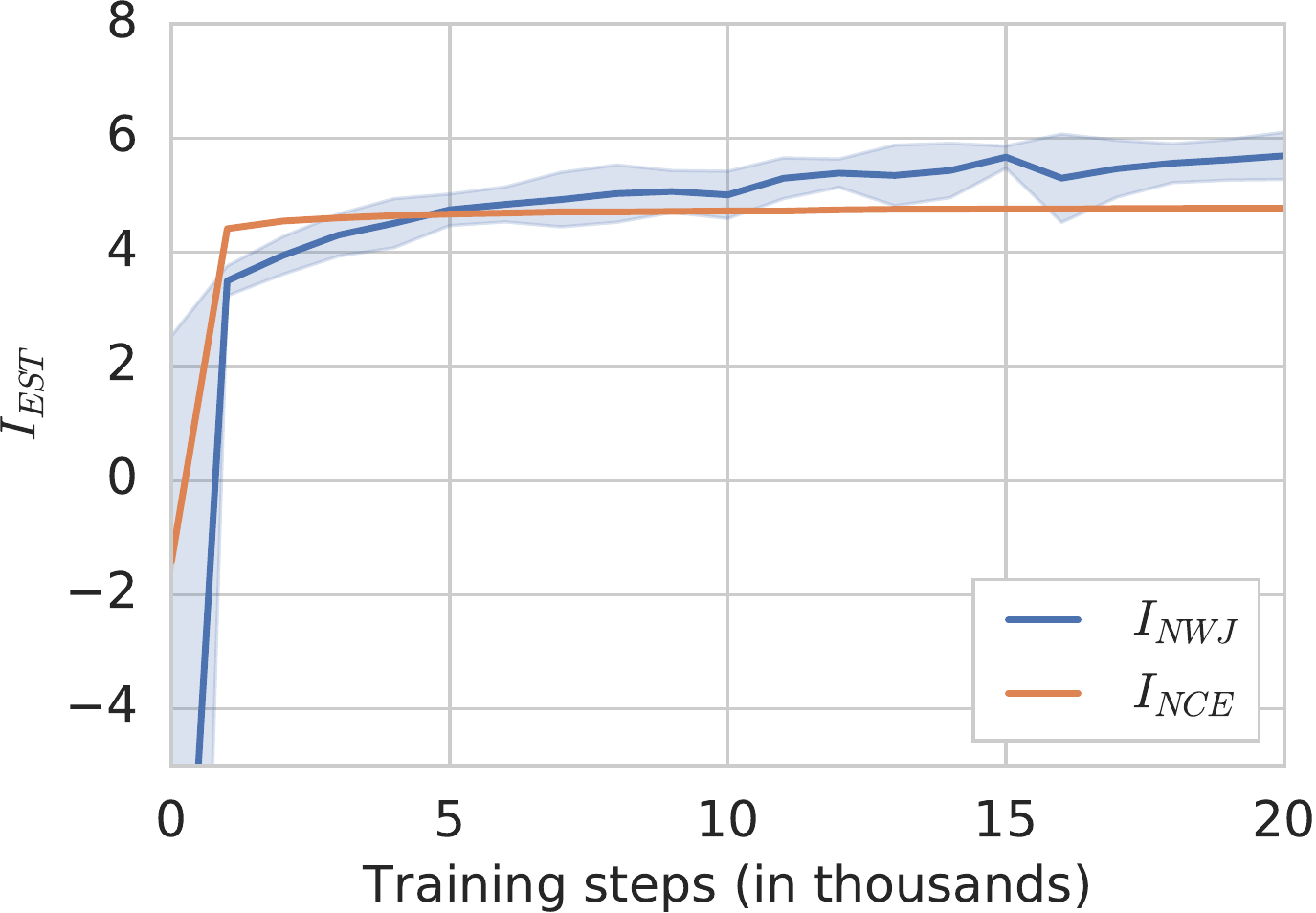}
    \caption{\:\label{subfig:realnvp_a}}
    \end{subfigure}
    \begin{subfigure}{0.32\textwidth}
    \includegraphics[width=\textwidth]{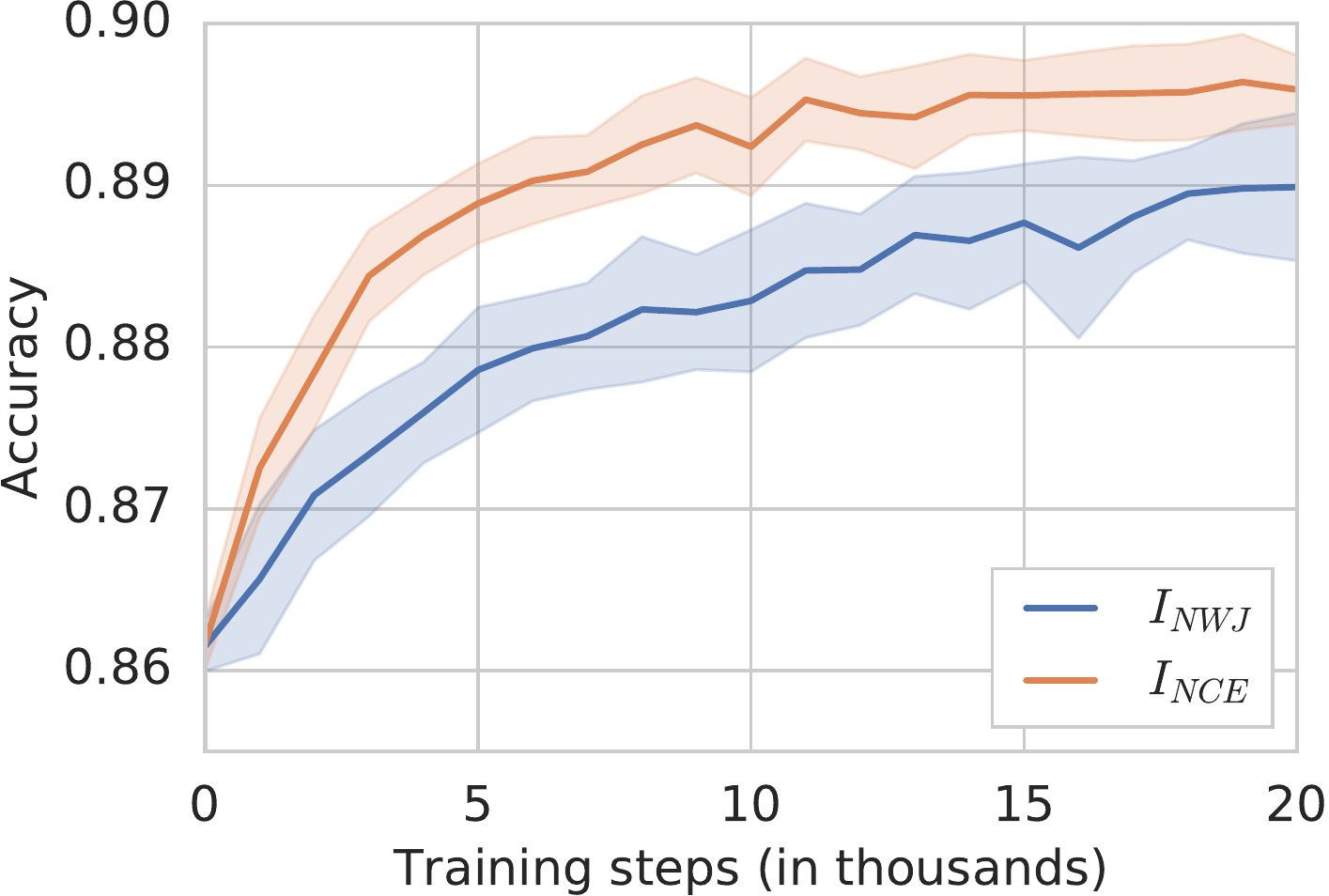}
    \caption{\:\label{subfig:realnvp_b}}
    \end{subfigure}
    \begin{subfigure}{0.32\textwidth}
    \includegraphics[width=\textwidth]{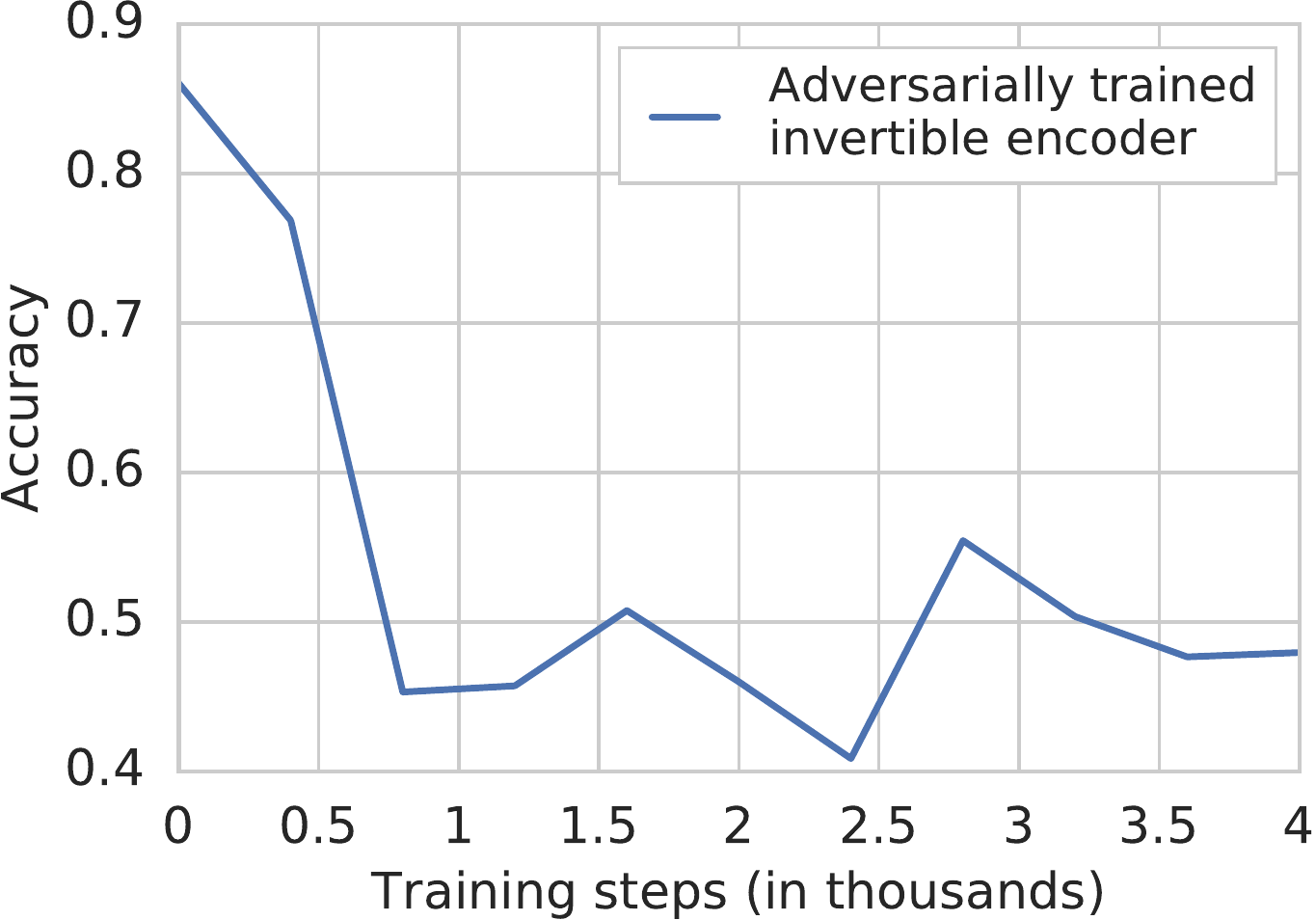}
    \caption{\:\label{subfig:realnvp_c}}
    \end{subfigure}
    \caption{(a, b) Maximizing $\iest$ over a family of invertible models. 
      We can see that during training the downstream classification performance improves (and the testing $\iest$ value increases), even though the true MI remains constant throughout. (c) Downstream classification accuracy of a different invertible encoder (with the same architecture) trained to have \emph{poor} performance. This demonstrates the existence of encoders that provably maximize MI yet have bad downstream performance.
    }
    \label{fig:realnvp}
    \vspace{-3mm}
\end{figure}

\subsection{Large MI is not predictive of downstream performance} \label{sec:lossexp}
We start by investigating the behavior of $\infonce$ and $\inwj$ when $g_1$ and $g_2$ are parameterized to be always invertible. Hence, for any choice of the encoder parameters, the MI is constant, i.e.\ $I(g_1(X^{(1)}); g_2(X^{(2)})) = I(X^{(1)};X^{(2)})$ for all $g_1, g_2$.
This means that if we could exactly compute the MI, any parameter choice would be a global maximizer and thus the gradients vanish everywhere.\footnote{In the context of continuous distributions and invertible representation functions $g$ the InfoMax objective might be infinite. \citet{bell1995information} suggest to instead maximize the entropy of the representation. In our case the MI between the two views is finite as the two halves are not deterministic functions of each another.} However, as we will empirically show, the estimators we consider are biased and prefer those settings which yield representations useful for the downstream classification task.

\textbf{Maximized MI and improved downstream performance}\quad We model $g_1$ and $g_2$ using the invertible RealNVP architecture~\citep{dinh2016density}. We use a total of 30 coupling layers, and each of them computes the shift using a separate MLP with two ReLU hidden layers, each with 512 units.

Figure~\ref{fig:realnvp} shows the testing value of $\iest$ and the testing accuracy on the classification task. Despite the fact that MI is maximized by any instantiation of $g_1$ and $g_2$, $\iest$ and downstream accuracy increase during training, implying that the estimators provide gradient feedback leading to a representation useful for linear classification. This confirms our hypothesis that the estimator biases the encoders towards solutions suitable to solve the downstream linear classification task. 

The previous experiment demonstrated that among many invertible encoders, all of which are globally optimal MI maximizers, some give rise to improved linear classification performance over raw pixels, and maximizing $\infonce$ and $\inwj$ yields such encoders. Next we demonstrate that for the same invertible encoder architecture there are model parameters for which linear classification performance is \emph{significantly worse} than using raw pixels, despite also being globally optimal MI maximizers.

\begin{figure}[t]
    \vspace{-2mm}
    \centering
    \begin{subfigure}{0.32\textwidth}
      \includegraphics[width=\textwidth]{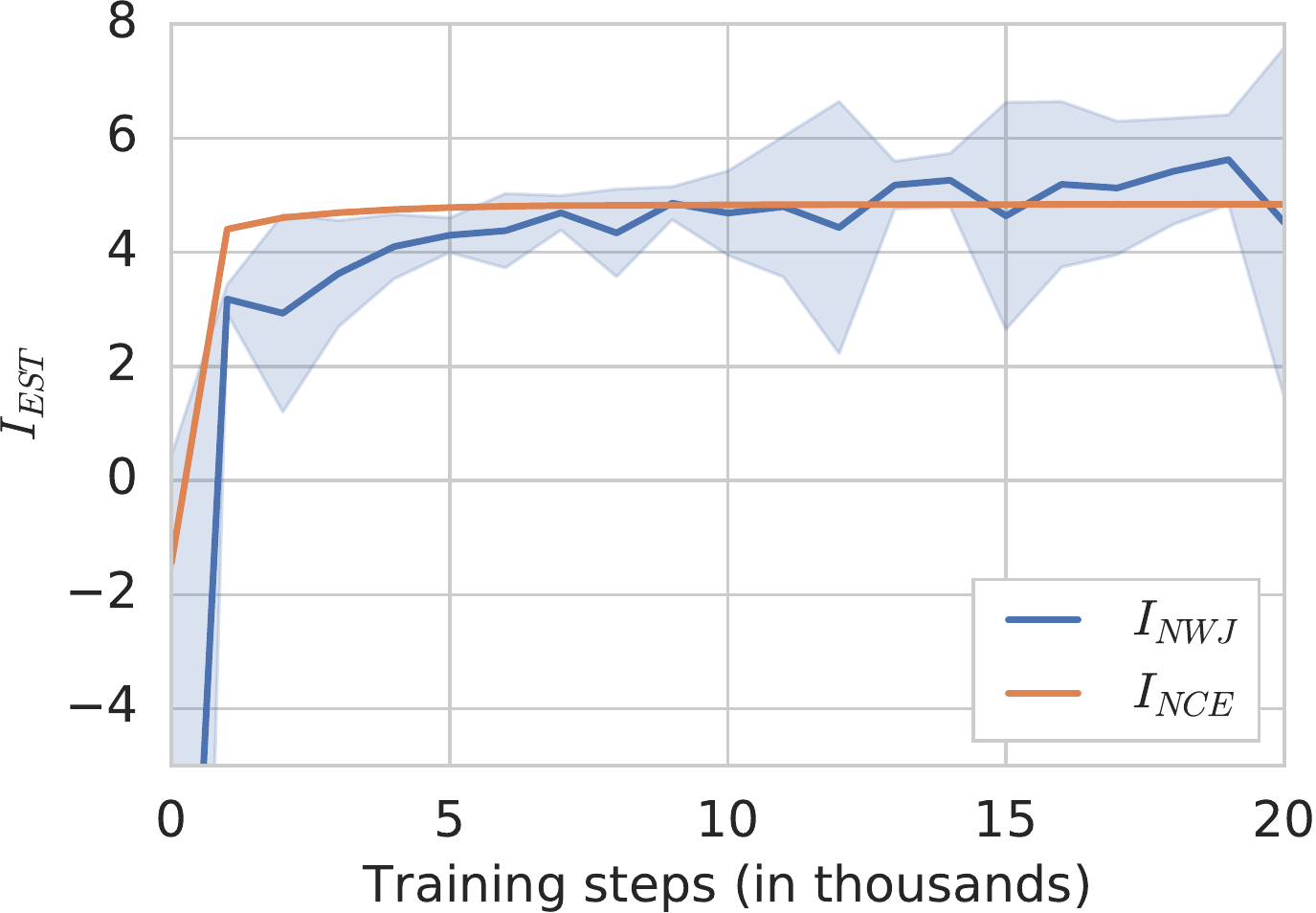}
      \caption{\:\label{fig:non_inv_sub1}}
    \end{subfigure}
    \begin{subfigure}{0.32\textwidth}
      \includegraphics[width=\textwidth]{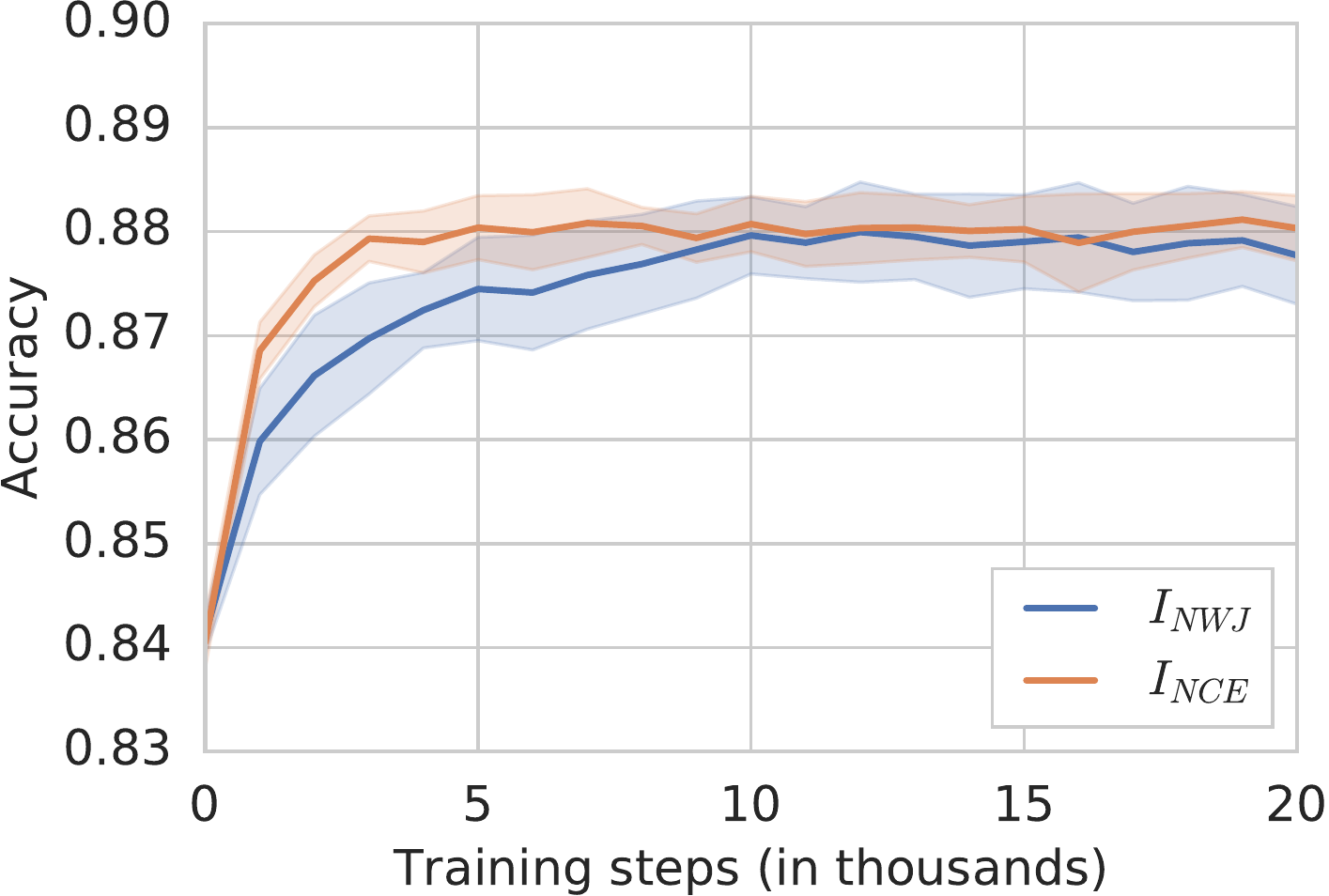}
      \caption{\:\label{fig:non_inv_sub2}}
    \end{subfigure}
    \begin{subfigure}{0.32\textwidth}
      \includegraphics[width=\textwidth]{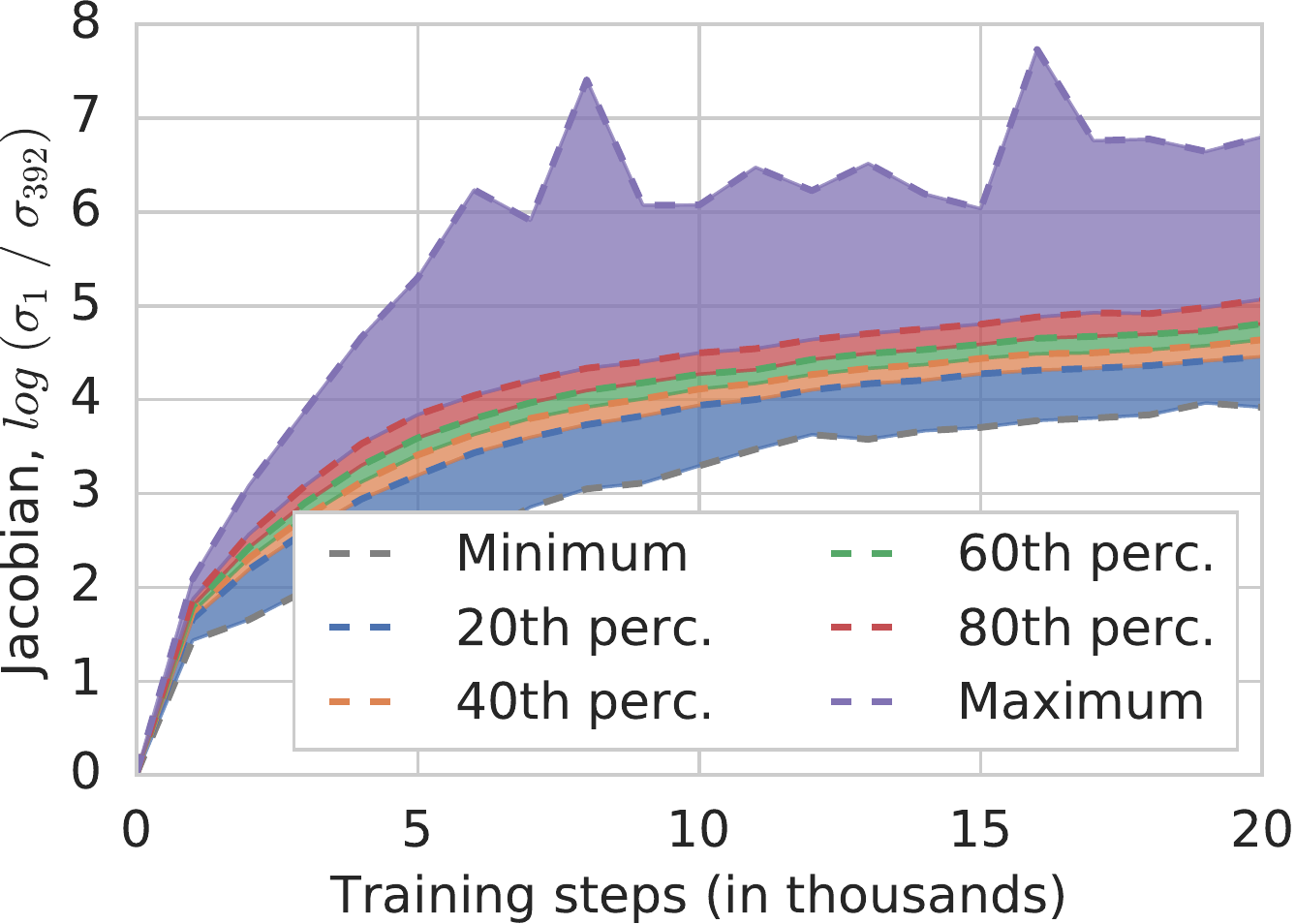}
      \caption{\:\label{fig:non_inv_sub3}}
    \end{subfigure}
    \vspace{-2mm}
    \caption{Maximizing $\iest$ using a network architecture that can realize both invertible and non-invertible functions. (a, b) As $\iest$ increases, the linear classification testing performance increases. (c) Meanwhile, the condition number of Jacobian evaluated at inputs randomly sampled from the data distribution deteriorates, i.e. $g_1$ becomes increasingly ill-conditioned (lines represent 0th, 20th, …, 100th percentiles for $\infonce$, the corresponding figure for $\inwj$ can be found in Appendix~\ref{app:figure}; the empirical distribution is obtained by randomly sampling $128$ inputs from the data distribution, computing the corresponding condition numbers, and aggregating them across runs).}
    \label{fig:mlp_shortcuts}
    \vspace{-3mm}
\end{figure}

\textbf{Maximized MI and worsened downstream performance}\quad
The goal is to learn a (bijective) representation maximizing MI such that the optimal linear classifier performs poorly; we achieve this by jointly training a representation and classifier in an adversarial fashion (a separate classifier is trained for the evaluation), without using a MI estimator. 
Intuitively, we will train the encoder to make the classification task for the linear layer as hard as possible. The experimental details are presented in Appendix~\ref{app:adversarial}. Figure~\ref{subfig:realnvp_c} shows the result of one such training run, displaying the loss of a separately trained classifier on top of the frozen representation.
At the beginning of training the network is initialized to be close to the identity mapping, and as such achieves the baseline classification accuracy corresponding to raw pixels.
All points beyond this correspond to invertible feature maps with worse classification performance, despite still achieving globally maximal MI.

Alternatively, the following thought experiment would yield the same conclusion: Using a lossless compression algorithm (e.g.\ PNG) for $g_1$ and $g_2$ also satisfies $I(g_1(X^{(1)}); g_2(X^{(2)})) = I(X^{(1)};X^{(2)})$. Yet, performing linear classification on the raw compressed bit stream $g_1(X^{(1)})$ will likely lead to worse performance than the baseline in pixel space. The information content alone is not sufficient to guarantee a useful \emph{geometry} in the representation space.

We next investigate the behavior of the model if we use a network architecture that can model both invertible and non-invertible functions. We would like to understand whether $\iest$ prefers the network to remain bijective, thus maximizing the true MI, or to ignore part of the input signal, which can be beneficial for representation learning.

\textbf{Bias towards hard-to-invert encoders}\quad We use an MLP architecture with $4$ hidden layers of the same dimension as the input, and with a skip connection added to each layer (hence by setting all weights to $0$ the network becomes the identity function).
As quantifying invertibility is hard, we analyze the condition number, i.e.\ the ratio between the largest and the smallest singular value, of the Jacobian of $g_1$: By the implicit function theorem, the function is invertible if the Jacobian is non-singular.\footnote{Formally, $g_1$ is invertible as long as the condition number of the Jacobian is finite. Numerically, inversion becomes harder as the condition number increases.}
However, the data itself might lie on a low-dimensional manifold, so that having a singular Jacobian is not necessarily indicative of losing invertibility on the support of the data distribution. 
To ensure the support of the data distribution covers the complete input space, we corrupt $X^{(1)}$ and $X^{(2)}$ in a coupled way by adding to each the \emph{same} 392-dimensional random vector, whose coordinates are sampled (independently of $X^{(1)}, X^{(2)}$) from a normal with standard deviation $0.05$ (the standard deviation of the pixels themselves is 0.3).
Hence, non-invertible encoders $g_1,g_2$ do not maximize $I(g_1(X^{(1)}); g_2(X^{(2)}))$.
\footnote{This would not necessarily be true if the noise were added in an uncoupled manner, e.g. by drawing it independently for $X^{(1)}$ and $X^{(2)}$, as the MI between the two noise vectors is $0$ in that case.}
As a reference point, the linear classification accuracy from pixels drops to about 84\% due to the added noise.

In Figure \ref{fig:mlp_shortcuts} we can see that the $\iest$ value and the downstream accuracy both increase during training, as before. Moreover, even though $g_1$ is initialized very close to the identity function (which maximizes the true MI), the condition number of its Jacobian evaluated at inputs randomly sampled from the data-distribution steadily deteriorates over time, suggesting that in practice (i.e.\ numerically) inverting the model becomes increasingly hard.
It therefore seems that the bounds we consider favor hard-to-invert encoders, which heavily attenuate part of the noise (as the support of the noise is the entire input space), over  well conditioned encoders (such as the identity function at initialization), which preserve the noise and hence the entropy of the data well.

\subsection{Higher capacity critics can lead to worse downstream performance}\label{sec:criticarch}

\begin{figure}[t!]
    \centering
    \includegraphics[width=0.32\textwidth]{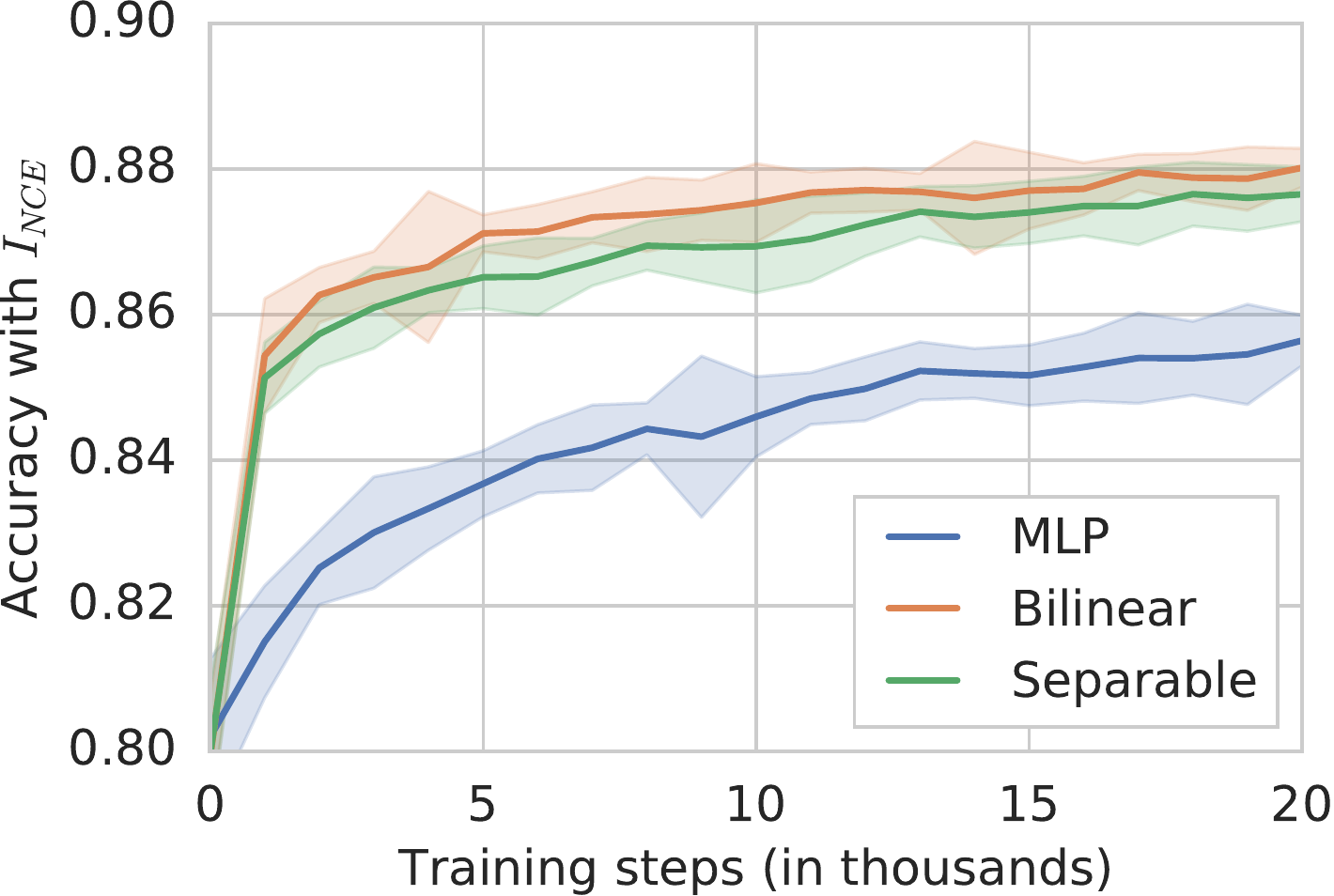}
    \includegraphics[width=0.32\textwidth]{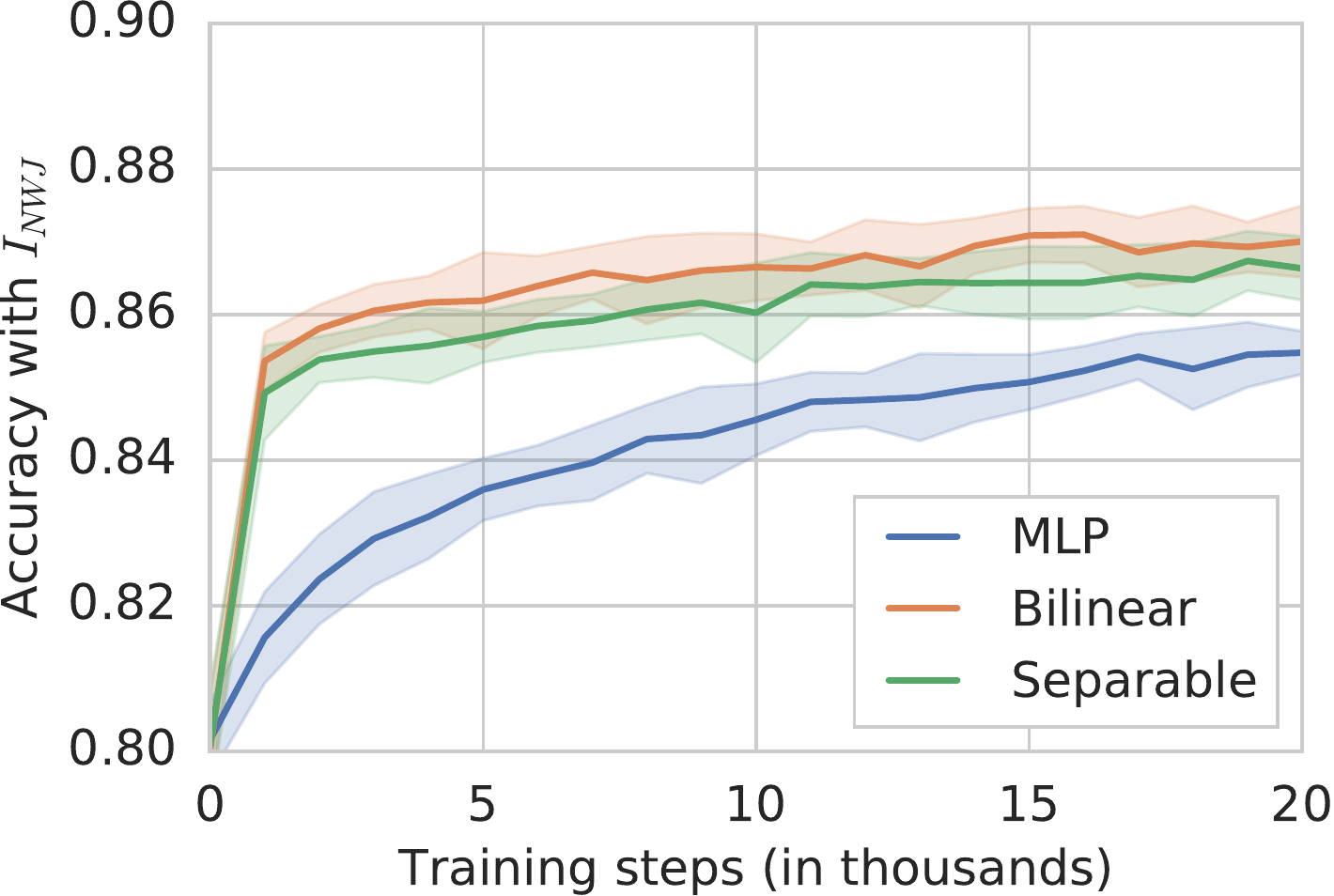}
    \includegraphics[width=0.3\textwidth]{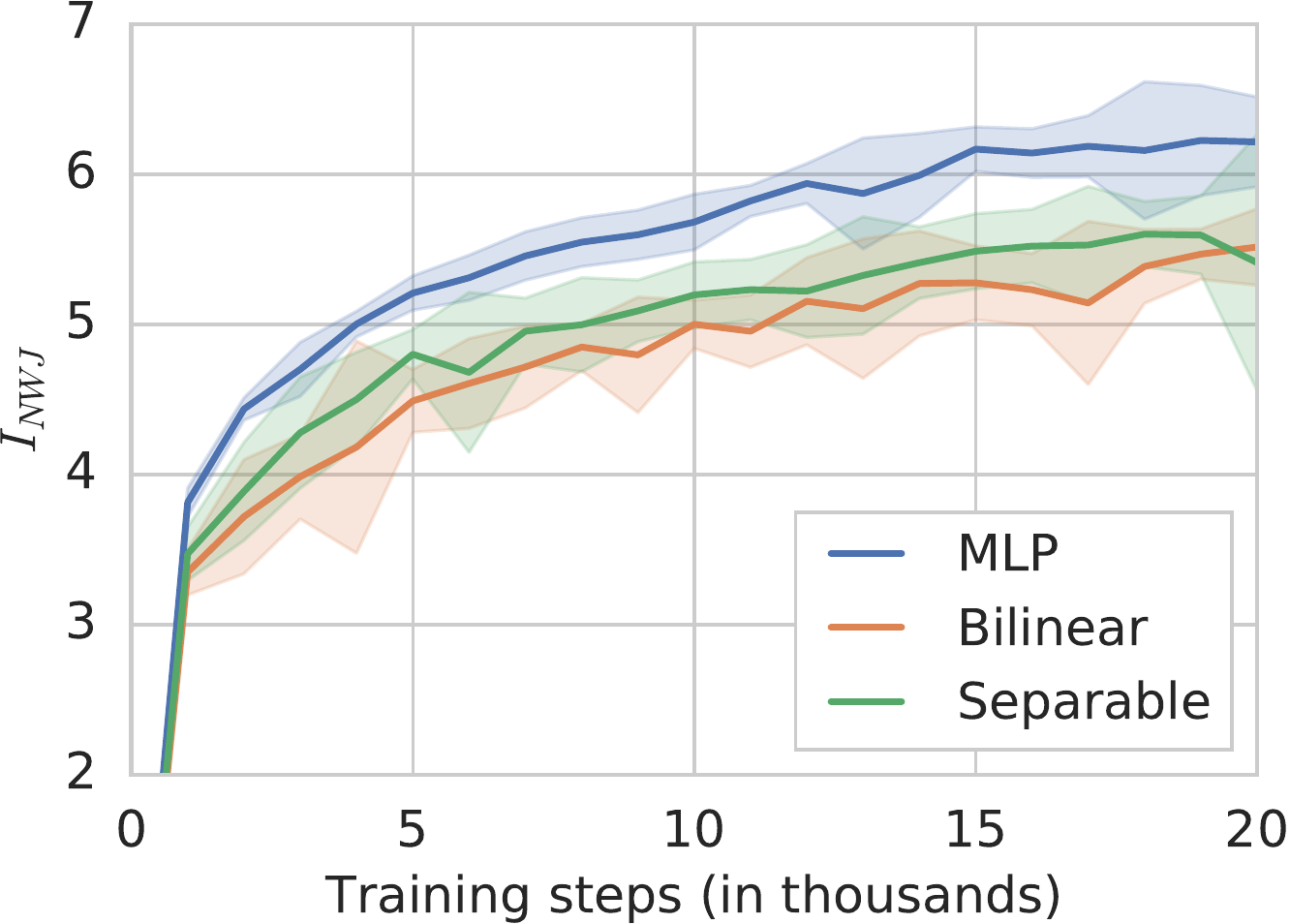}
    \caption{\looseness-1 Downstream testing accuracy for $\infonce$ and $\inwj$, and testing $\inwj$ value for MLP encoders $g_1, g_1$ and different critic architectures (the testing $\infonce$ curve can be found in Appendix~\ref{app:figure}). Bilinear and separable critics lead to higher downstream accuracy than MLP critics, while reaching lower $\inwj$.}\label{fig:critic_arch}
    \vspace{-5mm}
\end{figure}

In the previous section we have established that MI and downstream performance are only loosely connected. 
Clearly, maximizing MI is not sufficient to learn good representations and there is a non-trivial interplay between the architectures of the encoder, critic, and the underlying estimators.
In this section, we will focus on how one of these factors, namely the critic architecture, impacts the quality of the learned representation. Recall that it determines how the estimators such as $\infonce$ and $\mine$ distinguish between samples from the joint distribution $p(x,y)$ and the product of the marginals $p(x)p(y)$, and thereby determines the tightness on the lower bound.
A higher capacity critic should allow for a tighter lower-bound on MI~\citep{belghazi2018mine}. Furthermore, in the context of representation learning where $f$ is instantiated as a neural network, the critic provides gradient feedback to $g_1$ and $g_2$ and thereby shapes the learned representation.

\textbf{Looser bounds with simpler critics can lead to better representations}\quad We compare three critic architectures, a bilinear critic, a separable critic $f(x,y)=\phi_1(x)^\top \phi_2(y)$ ($\phi_1, \phi_2$ are MLPs with a single hidden layer with $100$ units and ReLU activations, followed by a linear layer with $100$ units; comprising $40$k parameters in total) and an MLP critic with a single hidden layer with $200$ units and ReLU activations, applied to the concatenated input $[x,y]$ ($40$k trainable parameters). Further, we use identical MLP architectures for $g_1$ and $g_2$ with two hidden layers comprising $300$ units each, and a third linear layer mapping to a $100$-dimensional feature space. 

Figure \ref{fig:critic_arch} shows the downstream testing accuracy and the testing $\iest$ value as a function of the iteration (see Appendix~\ref{app:cifar10results} for the corresponding results on CIFAR10). It can be seen that for both lower bounds, representations trained with the MLP critic barely outperform the baseline on pixel space, whereas the same lower bounds with bilinear and separable critics clearly lead to a higher accuracy than the baseline.
While the testing $\infonce$ value is close to the theoretically achievable maximum value for all critics, the testing $\mine$ value is higher for the MLP critic than for the separable and bilinear critics, resulting in a tighter bound on the MI. However, despite achieving the smallest $\mine$ testing value, the simple bilinear critic leads to a better downstream performance than the higher-capacity separable and MLP critics. 

A related phenomenon was observed in the context of variational autoencoders (VAEs) \citep{kingma2013auto}, where one maximizes a lower bound on the data likelihood: Looser bounds often yield better inference models, i.e.\ latent representations \citep{rainforth2018tighter}.

\subsection{Encoder architecture can be more important than the specific estimator} \label{sec:encoderarch}
We will now show that the encoder architecture is a critical design choice and we will investigate its effect on the learned representation. We consider the same MLP architecture ($238$k parameters) as in Section~\ref{sec:criticarch}, as well as a ConvNet architecture comprising two convolution layers (with a  $5\times5$ kernel, stride of $2$, ReLU activations, and $64$ and $128$ channels, respectively; $220$k parameters), followed by spatial average pooling and a fully connected layer. Before the average pooling operation we apply layer normalization~\citep{ba2016layer} which greatly reduces the variance of $\mine$.\footnote{LayerNorm avoids the possibility of information leakage within mini-batches that can be induced through batch normalization, potentially leading to poor performance~\citep{henaff2019data}.} To ensure that both network architectures achieve the same  lower bound $\iest$ on the MI, we minimize $L_t(g_1, g_2) = |\iest(g_1(X^{(1)}); g_1(X^{(2)})) - t|$ instead of solving \eqref{eq:newinfomax}, for two different values $t=2,4$.

Figure \ref{fig:encoder_arch} shows the downstream testing accuracy as a function of the training iteration (see Appendix~\ref{app:cifar10results} for the corresponding results on CIFAR10). It can be seen in the testing loss curves in Appendix~\ref{app:figure} that for both architectures and estimators the objective value after $7$k iterations matches the target $t$ (i.e., $L_t(g_1, g_2)\approx 0$) which implies that they achieve the same lower-bound on the MI. Despite matching lower bounds, ConvNet encoders lead to clearly superior classification accuracy, for both $\infonce$ and $\mine$. Note that, in contrast, the MLP and ConvNet architectures 
trained end-to-end in supervised fashion both achieve essentially the same testing accuracy of about $94\%$.

In the context of VAEs, \citet{alemi2018fixing} similarly observed that models achieving the same evidence lower bound value can lead to vastly different representations depending on the employed encoder architecture, and do not necessarily capture useful information about the data~\citep{tschannen2018recent,blau2019rethinking}.

\begin{figure}[t!]
  \centering
  \begin{subfigure}{0.32\textwidth}
    \includegraphics[width=\textwidth]{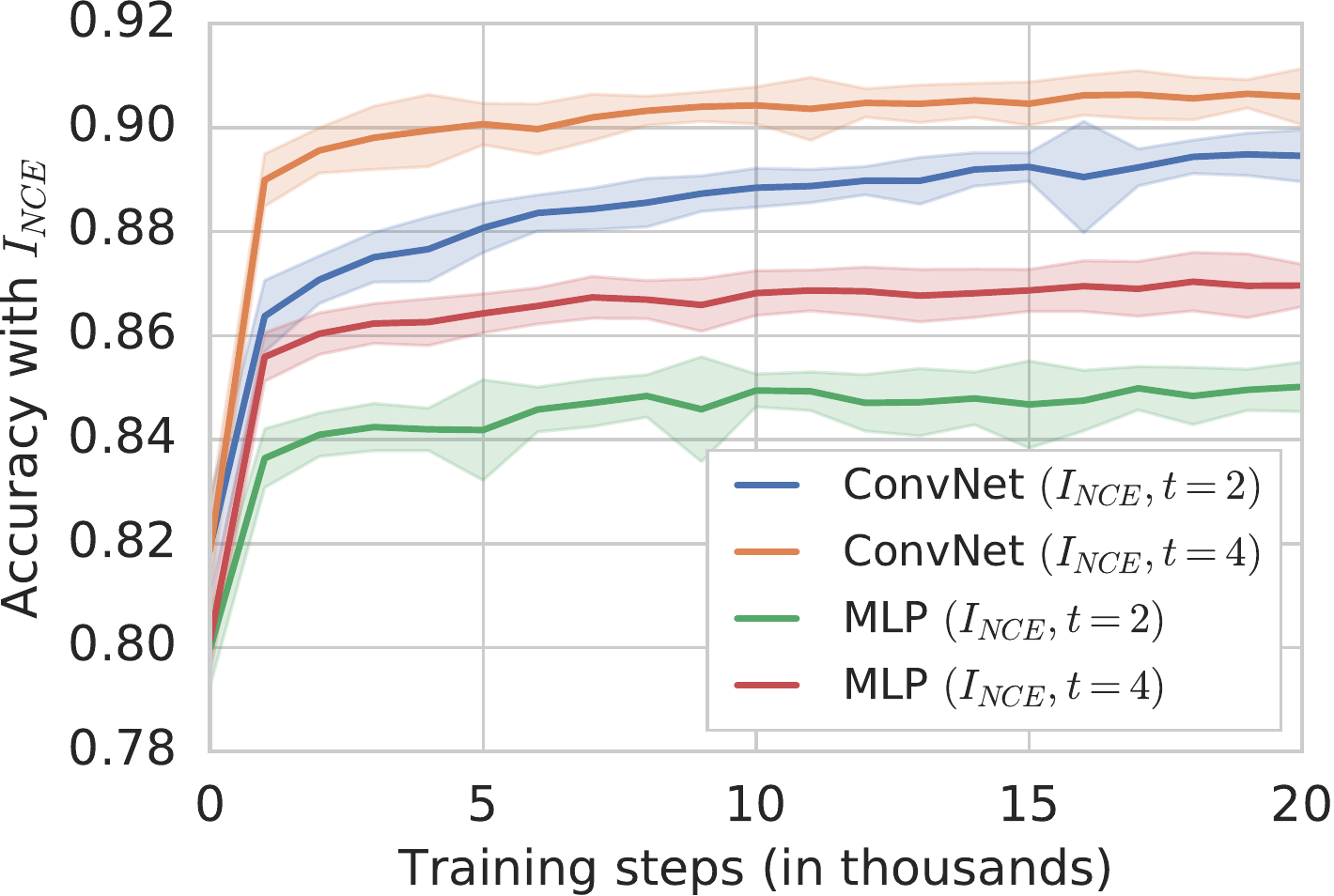}
    \caption{\:\label{sub1}}
  \end{subfigure}
  \begin{subfigure}{0.32\textwidth}
    \includegraphics[width=\textwidth]{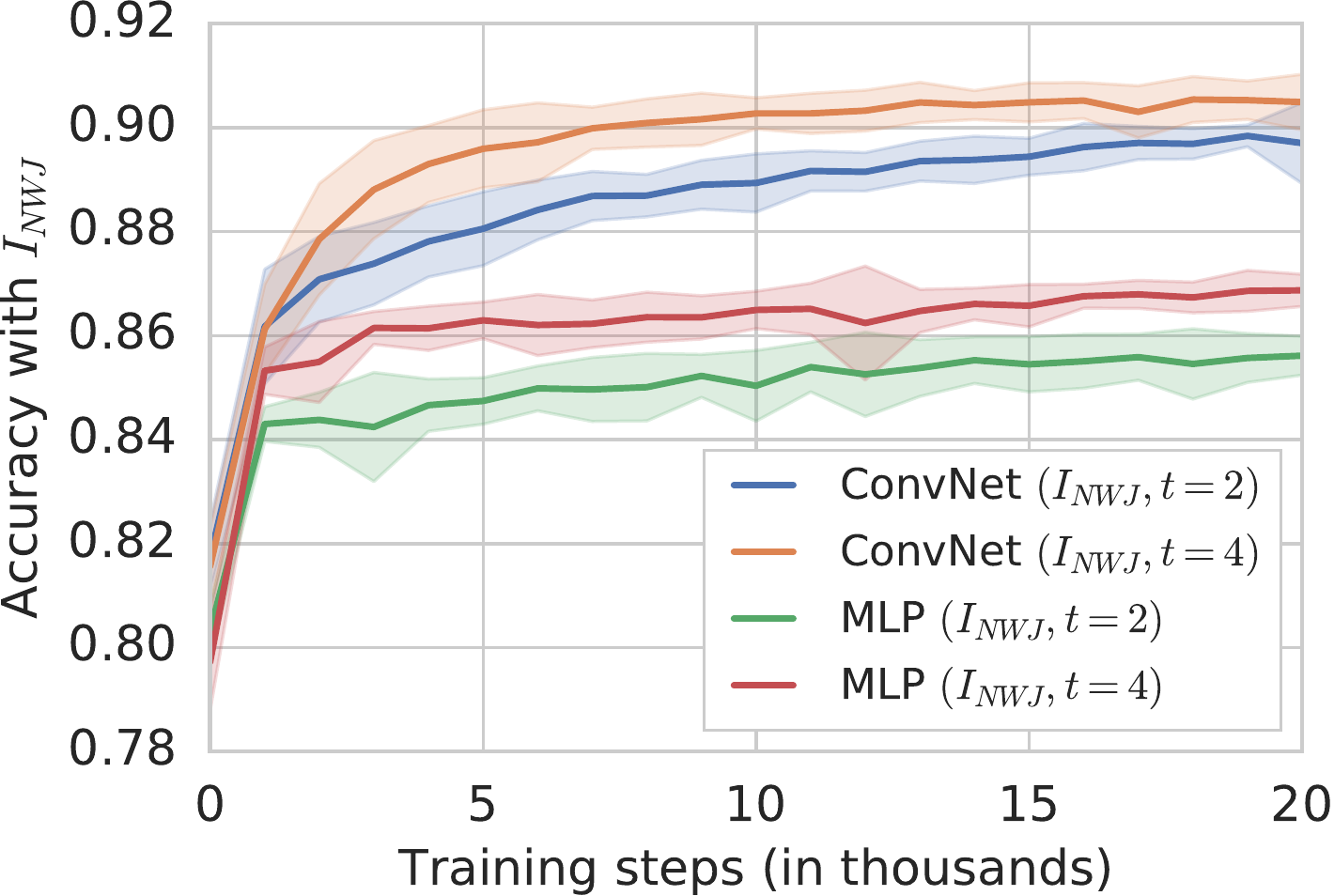}
    \caption{\:\label{sub2}}
  \end{subfigure}
  \begin{subfigure}{0.32\textwidth}
    \includegraphics[width=\textwidth]{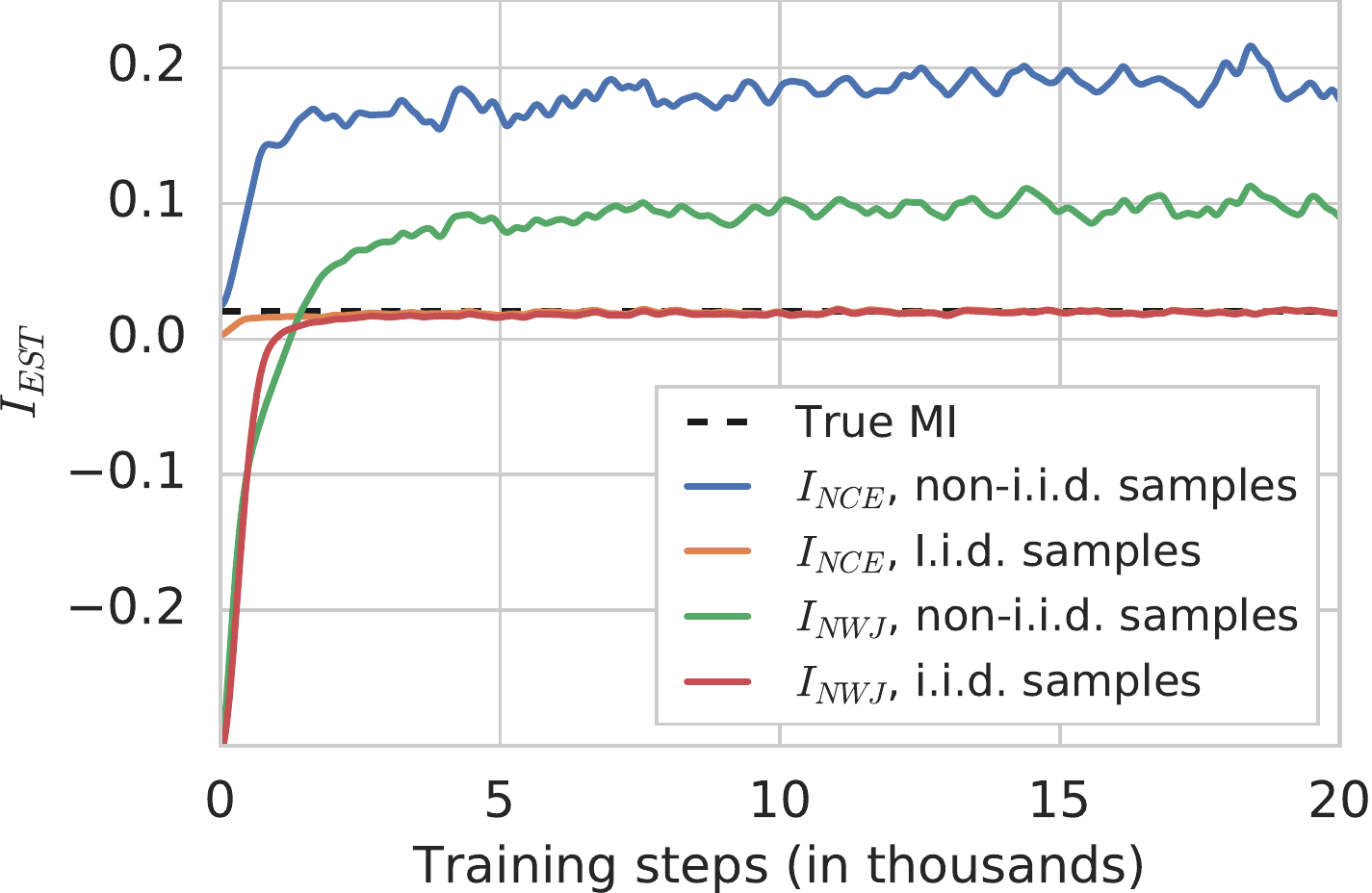}
    \caption{\:\label{fig:non_iid}}
  \end{subfigure}
  \caption{(a, b) Downstream testing accuracy for different encoder architectures and MI estimators, using a bilinear critic trained to match a given target $\iest$ of $t$ (we minimize $L_t(g_1, g_2) = |\iest(g_1(X^{(1)}); g_1(X^{(2)})) - t|$; loss curves can be found in Appendix~\ref{app:figure}). For a given estimator and $t$, ConvNet encoders clearly outperform MLP encoders in terms of downstream testing accuracy. (c) Estimating MI from \emph{i.i.d.}~and non-\emph{i.i.d.}~samples in a synthetic setting (Section~\ref{sec:triplet}). If negative samples are not drawn \emph{i.i.d.}, both $\infonce$ and $\inwj$ estimators can be \emph{greater} than the true MI. Despite being commonly justified as a lower bound on MI, $\infonce$ is often used in the non-\emph{i.i.d.}~setting in practice.}
  \label{fig:encoder_arch}
  \vspace{-3mm}
\end{figure}

\section{Connection to deep metric learning and triplet losses}\label{sec:triplet}

In the previous section we empirically demonstrated that there is a disconnect between approximate MI maximization and representation quality. However, many recent works have applied the $\infonce$ estimator to obtain state-of-the-art results in practice. We provide some insight on this conundrum by connecting $\infonce$ to a popular triplet ($k$-plet) loss known in the deep metric learning community.

\textbf{The metric learning view}\quad
Given sets of triplets, namely an \emph{anchor point} $x$, a positive instance $y$, and a negative instance $z$, the goal is to learn a representation $g(x)$ such that the distances (i.e., $\ell_2$) between $g(x)$ and $g(y)$ is smaller than the distance between $g(x)$ and $g(z)$, for each triplet. In the supervised setting, the positive instances are usually sampled from the same class, while the negative instances are sampled from any other class. A major focus in deep metric learning is how to perform \emph{(semi-)hard positive mining} --- we want to present non-trivial triplets to the learning algorithm which become more challenging as $g$ improves. Natural extensions to the unsupervised setting can be obtained by exploiting the structure present in the input data, namely spatial (e.g. patches from the same image should be closer than patches from different images) and temporal information (temporally close video frames should be encoded closer than the ones which are further away in time)~\citep{hoffer2015deep}.

\textbf{Connection to InfoNCE}\quad
The InfoNCE objective can be rewritten as follows:
\begin{align*}
    \infonce = \mathbb{E}\!\left[\frac{1}{K} \sum_{i=1}^{K} \log \frac{e^{f\left(x_{i}, y_{i}\right)}}{\frac{1}{K} \sum_{j=1}^{K} e^{f\left(x_{i}, y_{j}\right)}}\right]
   \! = \log K -\mathbb{E}\!\left[\frac{1}{K} \sum_{i=1}^{K} \log \left(1 +  \sum_{j\neq i} e^{f\left(x_{i}, y_{j}\right) - f\left(x_{i}, y_{i}\right)}\right)\!\right]\!.
\end{align*}
The derivation is presented in Appendix~\ref{eq:infonce_connection}. In the particular case that $x$ and $y$ take value in the same space and $f$ is constrained to be of the form $f(x,y) = \phi(x)^\top \phi(y)$, for some function $\phi$, this coincides (up to constants and change of sign) with the expectation of the \emph{multi-class K-pair} loss proposed in \citep[Eqn.\ (7)]{sohn2016improved}:
\begin{align}\vspace{-1mm}
    L_{\text{K-pair-mc}}\left( \{(x_i,y_i) \}_{i=1}^K, \phi \right) = \frac{1}{K}\sum_{i=1}^K \log\left( 1 + \sum_{j\not= i}e^{\phi(x_i)^\top \phi(y_j) - \phi(x_i)^\top\phi(y_i)} \right).\vspace{-1mm} \label{eq:npairloss}
\end{align}
Representation learning by maximizing $\infonce$ using a symmetric separable critic $f(x,y) = \phi(x)^\top \phi(y)$ and an encoder $g=g_1=g_2$ shared across views is thus equivalent to metric learning based on~\eqref{eq:npairloss}. When using different encoders for different views and asymmetric critics as employed by CPC, DeepInfoMax, and CMC one recovers asymmetric variants of \eqref{eq:npairloss}, see, e.g. \citep{yu2017cross, zhang2019large}. As a result, one can view~\eqref{eq:npairloss} as learning encoders with a parameter-less inner product critic, for which the MI lower-bound is very weak in general.

There are (at least) two immediate benefits of viewing recent representation learning methods based on MI estimators through the lens of metric learning. Firstly, in the MI view, using inner product or bilinear critic functions is sub-optimal since the critic should ideally be as flexible as possible in order to reduce the gap between the lower bound and the true MI. In the metric learning view, the inner product critic corresponds to a simple metric on the embedding space. The metric learning view seems hence in better accordance with the observations from Section~\ref{sec:criticarch} than the MI view. Secondly, it elucidates the importance of appropriately choosing the negative samples, which is indeed a critical component in deep metric learning based on triplet losses~\citep{norouzi2012hamming,schroff2015facenet}.

\textbf{InfoNCE and the importance of negative sampling}\quad
The negative sample mining issue also manifests itself in MI-based contrastive losses. In fact, while InfoNCE is a lower bound on MI if the negative samples are drawn from the true marginal distribution~\citep{pmlr-v97-poole19a}, i.e.
\begin{align*}
    I(X,Y) & \geq \mathbb{E}_{\prod_k p(x_k,y_k)} \frac{1}{K} \sum_{i=1}^K \left[\log\frac{e^{f(x_i,y_i)}}{\frac{1}{K} \sum_{j=1}^K e^{f(x_j,y_i)}}\right] \triangleq \infonce,
\end{align*}% \mathbb{E}_{p(x_{1:N},y_{1:N})}
we show that if the negative samples are drawn in a dependent fashion (corresponding to the $(x_i, y_i)$ being drawn identically but \emph{not independently}), the $\infonce$ estimator is in general neither a lower nor an upper bound on the true MI  $I(X,Y)$. We prove this in Appendix \ref{sec:infonce-inequality} and present empirical evidence here. Let $(X,Y) = Z + \epsilon$, where $Z \sim \mathcal{N}(0, \Sigma_Z)$ and $\epsilon \sim \mathcal{N}(0, \Sigma_\epsilon)$ are two-dimensional Gaussians.
We generate batches of data $(X_i, Y_i) = Z + \epsilon_i$ where each $\epsilon_i$ is sampled independently for each element of the batch, but $Z$ is sampled only once per batch. 
As such, $(X_i, Y_i)$ has the same marginal distribution for each $i$, but the elements of the batch \emph{are not independent}. Although we do not treat it theoretically, we also display results of the same experiment using the $\inwj$ estimator. The experimental details are presented in Appendix~\ref{app:noniid}. We observe in Figure~\ref{fig:encoder_arch}c that when using non-\emph{i.i.d.}~samples both the $\infonce$ and $\inwj$ values are larger than the true MI, and that when \emph{i.i.d.}~samples are used, both are lower bounds on the true MI. Hence, the connection to MI under improper negative sampling is no longer clear and might vanish completely.

Notwithstanding this fundamental problem, the negative sampling strategy is often treated as a design choice. In \citet{henaff2019data}, CPC is applied to images by partitioning the input image into patches. Then, MI (estimated by InfoNCE) between representations of patches and a \emph{context} summarizing several patches that are vertically above or below in the same image is minimized. Negative samples are obtained by patches from different images as well as patches from the \emph{same} image, violating the independence assumption. Similarly, \citet{oord2018representation} learn representations of speech using samples from a variety of speakers. It was found that using utterances from the same speaker as negative samples is more effective, whereas the ``proper'' negative samples should be drawn from an appropriate mixture of utterances from all speakers.

A common observation is that increasing the number of negative examples helps in practice ~\citep{hjelm2018learning,tian2019contrastive,bachman2019learning}. Indeed,~\citet{ma2018noise} show that $\infonce$ is consistent for any number of negative samples (under technical conditions), and~\citet{pmlr-v97-poole19a} show that the signal-to-noise ratio increases with the number of negative samples. On the other hand,~\citep{arora2019theoretical} have demonstrated, both theoretically and empirically, that increasing the number of negative samples does not necessarily help, and can even deteriorate the performance. The intricacies of negative sampling hence remain a key research challenge.
 
\section{Conclusion} \label{sec:conclusion}
Is MI maximization a good objective for learning good representations in an unsupervised fashion? Possibly, but it is clearly not sufficient. In this work we have demonstrated that, under the common linear evaluation protocol, maximizing lower bounds on MI as done in modern incarnations of the InfoMax principle can result in bad representations. We have revealed that the commonly used estimators have strong inductive biases and---perhaps surprisingly---looser bounds can lead to better representations. Furthermore, we have demonstrated that the connection of recent approaches to MI maximization might vanish if negative samples are not drawn independently (as done by some approaches in the literature). As a result, it is unclear whether the connection to MI is a sufficient (or necessary) component for designing powerful unsupervised representation learning algorithms. We propose that the success of these recent methods could be explained through the view of triplet-based metric learning and that leveraging advances in that domain might lead to further improvements. We have several suggestions for future work, which we summarize in the following.

\textbf{Alternative measures of information}\quad
We believe that the question of developing new notions of information suitable for representation learning should receive more attention. While MI has appealing theoretical properties,  it is clearly not sufficient for this task---it is hard to estimate, invariant to bijections and can result in suboptimal representations which do not correlate with downstream performance. Therefore, a new notion of information should account for both the amount of information stored in a representation and the geometry of the induced space necessary for good performance on downstream tasks. 
One possible avenue is to consider extensions to MI which explicitly account for the modeling power and computational constraints of the observer, such as the recently introduced $\mathcal{F}$-information~\citet{xu2020a}. Alternatively, one can investigate other statistical divergences to measure the discrepancy between $p(x, y)$ and $p(x)p(y)$. For example, using the Wasserstein distance leads to promising results in representation learning as it naturally enforces smoothness in the encoders~\citep{ozair2019wasserstein}.

\textbf{A holistic view}\quad 
We believe that any theory on measuring information for representation learning built on critics should explicitly take into account the function families one uses (e.g.\ that of the critic and estimator). Most importantly, we would expect some natural trade-offs between the amount of information that can be stored against how hard it is to extract it in the downstream tasks as a function of the architectural choices. While the distribution of downstream tasks is typically assumed unknown in representation learning, it might be possible to rely on weaker assumptions such as a family of invariances relevant for the downstream tasks.  Moreover, it seems that in the literature (i) the critics that are used to measure the information, (ii) the encoders, and (iii) the downstream models/evaluation protocol are all mostly chosen independently of each other. Our empirical results show that the downstream performance depends on the intricate balance between these choices and we believe that one should co-design them. This holistic view is currently under-explored and due to the lack of any theory or extensive studies to guide the practitioners.

\textbf{Going beyond the widely used linear evaluation protocol}\quad
While it was shown that learning good representations under the linear evaluation protocol can lead to reduced sample complexity for downstream tasks~\citep{arora2019theoretical}, some recent works~\citep{bachman2019learning,tian2019contrastive} report marginal improvements in terms of the downstream performance under a non-linear regime. Related to the previous point, it would hence be interesting to further explore the implications of the evaluation protocol, in particular its importance in the context of other design choices. We stress that a highly-nonlinear evaluation framework may result in better downstream performance, but it defeats the purpose of learning efficiently transferable data representations.

\textbf{Systematic investigations into design decisions that matter} \quad 
On the practical side, we believe that the link to metric learning could lead to new methods, that break away from the goal of estimating MI and place more weight on the aspects that have a stronger effect on the performance such as the negative sampling strategy. An example where the metric learning perspective led to similar methods as the MI view is presented by \citet{sermanet2018time}: They developed a multi-view representation learning approach for video data similar to CMC, but without drawing negative samples independently and seemingly without relying on the MI mental model to motivate their design choices. 

\section*{Acknowledgments}
We would like to thank Alex Alemi, Ben Poole, Olivier Bachem, and Alexey Dosovitskiy for inspiring discussions and comments on the manuscript. We are grateful for the general support and discussions from other members of Google Brain team in Zurich.

\bibliography{main}
\bibliographystyle{plainnat}

\clearpage

\appendix

\section*{Appendix}
\label{sec:appendix}
\section{Relation between \eqref{eq:newinfomax} and the InfoMax objective} \label{app:newinfomax}

\begin{proposition}\label{prop:newinfomax}
Let $X$ be a random variable and define $X_1=g_1(X)$ and $X_2=g_2(X)$ be arbitrary functions of $X$. Then
$I(X_1; X_2) \leq I\left(X; (X_1, X_2)\right)$.
\end{proposition}
\begin{proof}
Follows by two applications of the \emph{data processing inequality}, which states that for random variables $X$, $Y$ and $Z$ satisfying the Markov relation $X \rightarrow Y \rightarrow Z$, the inequality $I(X; Z) \leq I(X; Y)$ holds.

The first step is to observe that $X$, $X_1$ and $X_2$ satisfy the relation $X_1 \leftarrow X \rightarrow X_2$, which is Markov equivalent to $X_1 \rightarrow X \rightarrow X_2$ (in particular, $X_1$ and $X_2$ are conditionally independent given $X$). It therefore follows that $I(X_1; X_2) \leq I(X; X_1)$.
The second step is to observe that $X \rightarrow (X_1, X_2) \rightarrow X_1$ and therefore $I(X; X_1) \leq I(X; (X_1, X_2))$.

Combining the two inequalities yields $I(X_1; X_2) \leq I(X; (X_1, X_2))$, as required.
\end{proof}

\section{Experiment details: Adversarially trained encoder (Section~\ref{sec:lossexp})}\label{app:adversarial}
In the following, we present the details for training the invertible model from Section~\ref{sec:lossexp} adversarially. We model $g_1$ with the same RealNVP architecture as in the first experiment, and do not model $g_2$. On top of $g_1(X^{(1)})$ we add a linear layer mapping to $10$ outputs (i.e. logits).
The parameters of the linear layer trained by minimizing the cross-entropy loss with respect to the true label of $X$ from which $X^{(1)}$ is derived. Conversely, the parameters of the encoder $g_1$ are trained to minimize the cross-entropy loss with respect to a uniform probability vector over all 10 classes. We use the Adam optimizer with a learning rate of $10^{-4}$ for the parameters of the classifier and $10^{-6}$ for the parameters of the encoder, and perform $10$ classifier optimization steps per encoder step. Furthermore, in a warm-up phase we train the classifier for $1$k iterations before alternating between classifier and encoder steps.

\section{Connection between metric learning and InfoNCE} \label{eq:infonce_connection}
$\infonce$ can be rewritten as follows:
\begin{align*}
    \infonce &= \mathbb{E}\left[\frac{1}{K} \sum_{i=1}^{K} \log \frac{e^{f\left(x_{i}, y_{i}\right)}}{\frac{1}{K} \sum_{j=1}^{K} e^{f\left(x_{i}, y_{j}\right)}}\right] \\
    &= \mathbb{E}\left[\frac{1}{K} \sum_{i=1}^{K} \log \frac{1}{\frac{1}{K} \sum_{j=1}^{K} e^{f\left(x_{i}, y_{j}\right) - f\left(x_{i}, y_{i}\right)}}\right] \\
    &= \mathbb{E}\left[-\frac{1}{K} \sum_{i=1}^{K} \log \frac{1}{K} \sum_{j=1}^{K} e^{f\left(x_{i}, y_{j}\right) - f\left(x_{i}, y_{i}\right)}\right] \\
    &= \log K -\mathbb{E}\left[\frac{1}{K} \sum_{i=1}^{K} \log \left(1 +  \sum_{j\neq i} e^{f\left(x_{i}, y_{j}\right) - f\left(x_{i}, y_{i}\right)}\right)\right].
\end{align*}

\section{InfoNCE under non-i.i.d. sampling}\label{sec:infonce-inequality}

The proof that InfoNCE is a lower bound on MI presented in \citep{pmlr-v97-poole19a} makes crucial use of the assumption that the negative samples are drawn from the true marginal distribution. We briefly review this proof to highlight the importance of the negative sampling distribution. Their proof starts from the NWJ lower bound of the KL divergence, namely that for any  
function $\tilde f$ the following lower bound holds \citep{nguyen2010estimating, nowozin2016f}:
\begin{align} \label{eq:nwjlbapp}
    I(X; Y) &= D_{KL}(p(x, y)\,||\,p(x)p(y))  \geq \mathbb{E}_{p(x,y)}[\tilde f(x,y)] - e^{-1}\mathbb{E}_{p(x)p(y)}[e^{\tilde f(x,y)}].
\end{align}

Suppose that $(X_i, Y_i)_{i=1}^K$ are \emph{i.i.d.} draws from $p(x,y)$ and write $X_{1:K} = (X_1, X_2,\ldots,X_K)$. Then, for any $i$ we have that $I(X_{1:K}; Y_i) = I(X_i; Y_i) = I(X; Y)$. We thus have

\begin{align*}
    I(X; Y) = I(X_{1:K}; Y_i) & \geq \mathbb{E}_{p(x_i,y_i)\prod_{k \neq i}p(x_k)}[\tilde f(x_{1:K},y_i)] - e^{-1}\mathbb{E}_{p(y_i)\prod_k p(x_k)}[e^{\tilde f(x_{1:K},y_i)}], 
\end{align*}
where the equality follows from the assumption that the $(X_i, Y_i)_{i=1}^K$ are i.i.d.\ and the inequality is \eqref{eq:nwjlbapp} applied to $I(X_{1:K}; Y_i)$. 
In particular, taking $\tilde f(x_{1:K}, y_i) = 1 + \log\frac{e^{f(x_i,y_i)}}{\frac{1}{K} \sum_{j=1}^K e^{f(x_j,y_i)}}$ yields

\begin{align}\label{eq:inceinterapp}
    I(X,Y) &\! \geq 1\! + \mathbb{E}_{p(x_i,y_i)\prod_{k \neq i}p(x_k)}\left[\log\frac{e^{f(x_i,y_i)}}{\frac{1}{K} \sum_{j=1}^K e^{f(x_j,y_i)}}\right]\! - \mathbb{E}_{p(y_i) \prod_k p(x_k)}\left[\frac{e^{f(x_i,y_i)}}{\frac{1}{K} \sum_{j=1}^K e^{f(x_j,y_i)}}\right]\!.
\end{align}

This is then averaged over the $K$ samples $Y_i$, in which case the third term above cancels with the constant $1$ (all occurences of $y_i$ in the last term of \eqref{eq:inceinterapp} can be replaced with $y_1$ thanks to $(X_i, Y_i)$ being identically distributed), yielding the familiar $\infonce$ lower bound:

\begin{align} \label{eq:incefinalapp}
    I(X,Y) & \geq \mathbb{E}_{\prod_k p(x_k, y_k)} \frac{1}{K} \sum_{i=1}^K \left[\log\frac{e^{f(x_i,y_i)}}{\frac{1}{K} \sum_{j=1}^K e^{f(x_j,y_i)}}\right] = \infonce.
\end{align}

The point in this proof that makes use of the \emph{i.i.d.}\ assumption of the negative samples is in the equality $I(X_i,Y_i) = I(X_{1:K}, Y_i)$, which allowed us to leverage multiple samples when estimating the MI between two variables. 
If instead the negative samples are drawn in a dependent fashion (corresponding to the $(X_i, Y_i)$ being drawn identically but \emph{not independently}), we have $I(X_i,Y_i) \leq I(X_{1:K}, Y_i)$, though the remainder of the proof still holds, resulting in

\begin{align*}
    I(X,Y)  \leq \frac{1}{K}\sum_{i=1}^K I(X_{1:K}; Y_i) & \geq \mathbb{E}_{p(x_{1:K}, y_{1:K})} \frac{1}{K} \sum_{i=1}^K \left[\log\frac{e^{f(x_i,y_i)}}{\frac{1}{K} \sum_{j=1}^K e^{f(x_j,y_i)}}\right].
\end{align*}

Therefore the resulting $\infonce$ estimator is neither a lower nor an upper bound on the true MI $I(X,Y)$.

\section{Experiment details: Non-i.i.d. sampling (Section~\ref{sec:triplet})}\label{app:noniid}
Recall that $(X,Y) = Z + \epsilon$. We use $Z \sim \mathcal{N}(0, \Sigma_Z)$ and $\epsilon \sim \mathcal{N}(0, \Sigma_\epsilon)$, where
\begin{equation*}
    \Sigma_Z=\begin{pmatrix}1 & -0.5 \\ -0.5 & 1 \end{pmatrix} \qquad \text{and} \qquad \Sigma_\epsilon = \begin{pmatrix}1 & 0.9 \\ 0.9 & 1 \end{pmatrix}.
\end{equation*}
Batches of data are obtained as $(X_i, Y_i) = Z + \epsilon_i$ where each $\epsilon_i$ is sampled independently for each element of the batch, but $Z$ is sampled only once per batch.
The true MI $I(X,Y)$ can be calculated analytically since $(X,Y)$ is jointly Gaussian with known covariance matrix $\Sigma_Z + \Sigma_\epsilon$: For two univariate random variables $(X,Y)$ that are jointly Gaussian with covariance $\Sigma$ the MI can be written as 
\begin{equation*}
    I(X,Y) = -\frac{1}{2}\log(1-\frac{\Sigma_{12}\Sigma_{21}}{\Sigma_{11}\Sigma_{22}}).
\end{equation*}
This can be derived using the decomposition $I(X,Y) = H(X) + H(Y) - H(X,Y)$ and the analytic expression for the entropy $H$ of a Gaussian. 

We compare the same setting trained using \emph{i.i.d.}~sampled pairs $(X_i, Y_i)$ as a baseline.
We parametrize the critic as a MLP with 5 hidden layers, each with 10 units and ReLU activations, followed by a linear layer and maximize $\infonce$ using these non-\emph{i.i.d.}~samples with batch size 128. Note that if a batch size of $K$ is used, the bound $\infonce \leq \log K$ always holds. We used $K$ sufficiently large so that $I(X,Y) \leq \log K$ to avoid $\infonce$ trivially lower bounding the true MI. 

\clearpage

\section{Additional Figures}\label{app:figure}
\FloatBarrier
\begin{figure}[h!]
    \centering
    \includegraphics[width=0.32\textwidth]{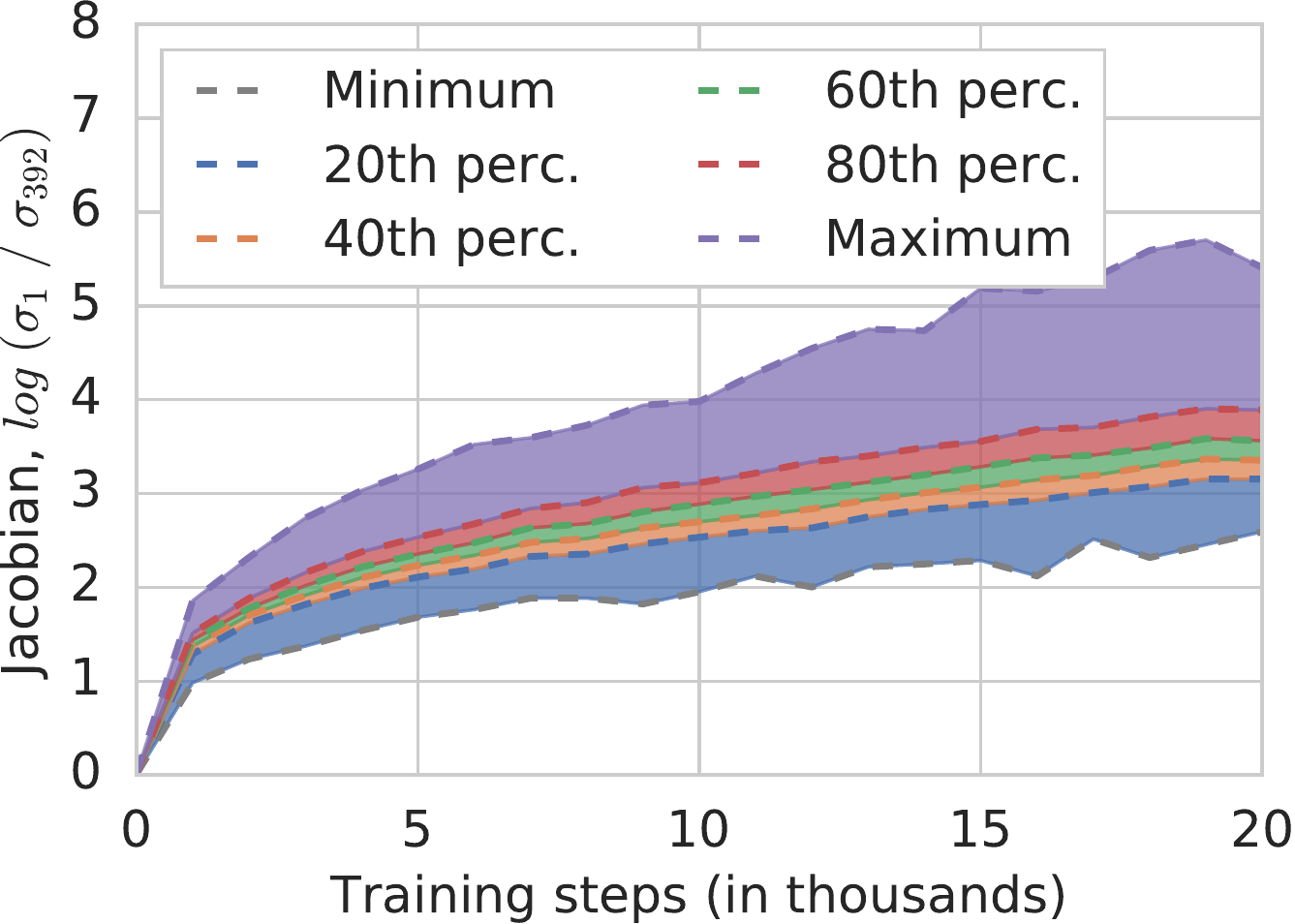}
    \caption{Additional plot Section~\ref{sec:lossexp}: The condition number of the Jacobian evaluated at inputs randomly sampled from the data distribution deteriorates, i.e. $g_1$ becomes increasingly ill-conditioned (lines represent 0th, 20th, …, 100th percentiles for $\inwj$; the empirical distribution is obtained by randomly sampling $128$ inputs from the data distribution, computing the corresponding condition numbers, and aggregating them across runs).}
\end{figure}

\begin{figure}[h!]
    \centering
    \includegraphics[width=0.35\textwidth]{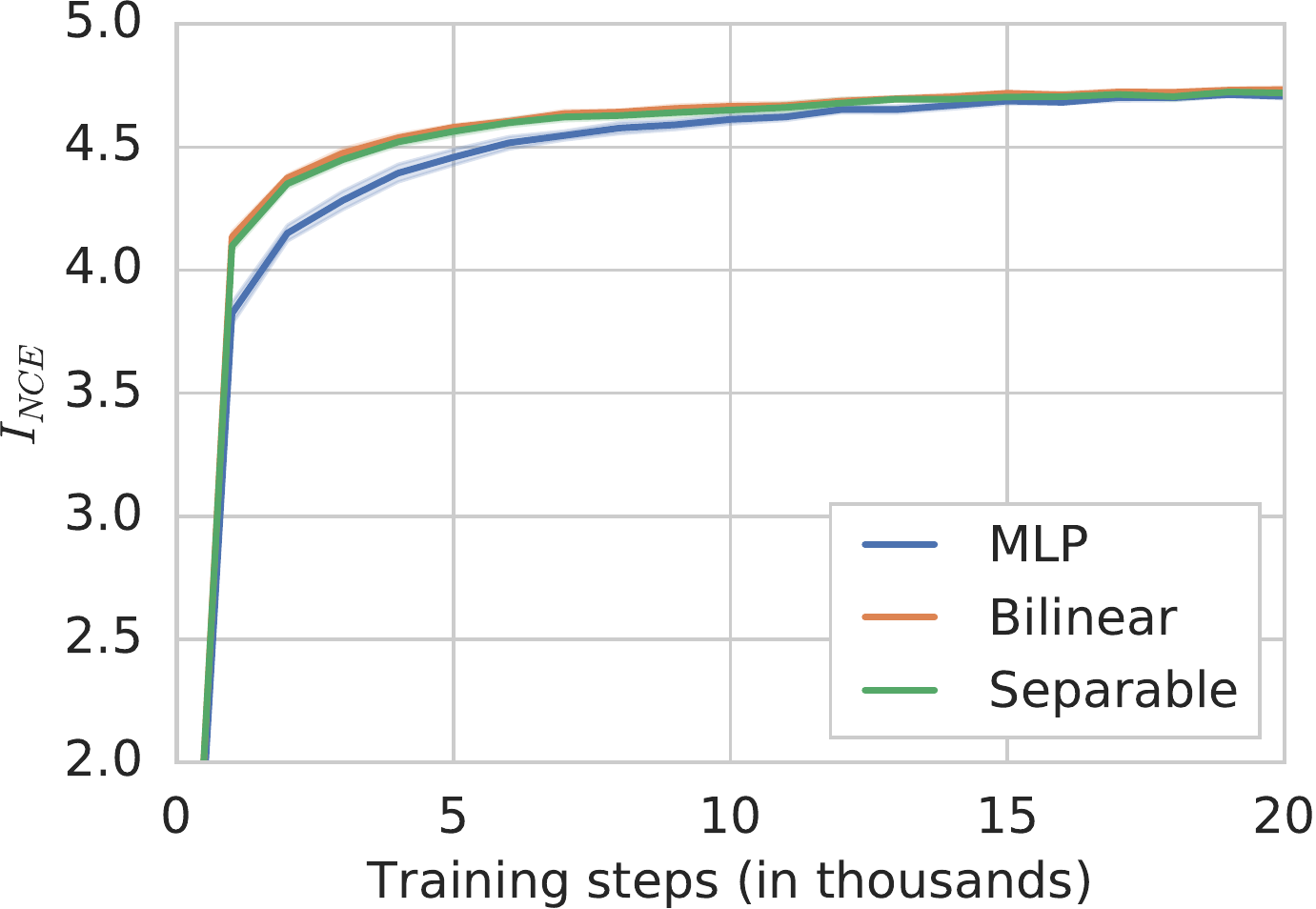}
    \caption{Additional plot Section~\ref{sec:criticarch}: Testing $\infonce$ value for MLP encoders $g_1, g_1$ and different critic architectures.}
\end{figure}

\begin{figure}[h!]
    \centering
    \includegraphics[width=0.35\textwidth]{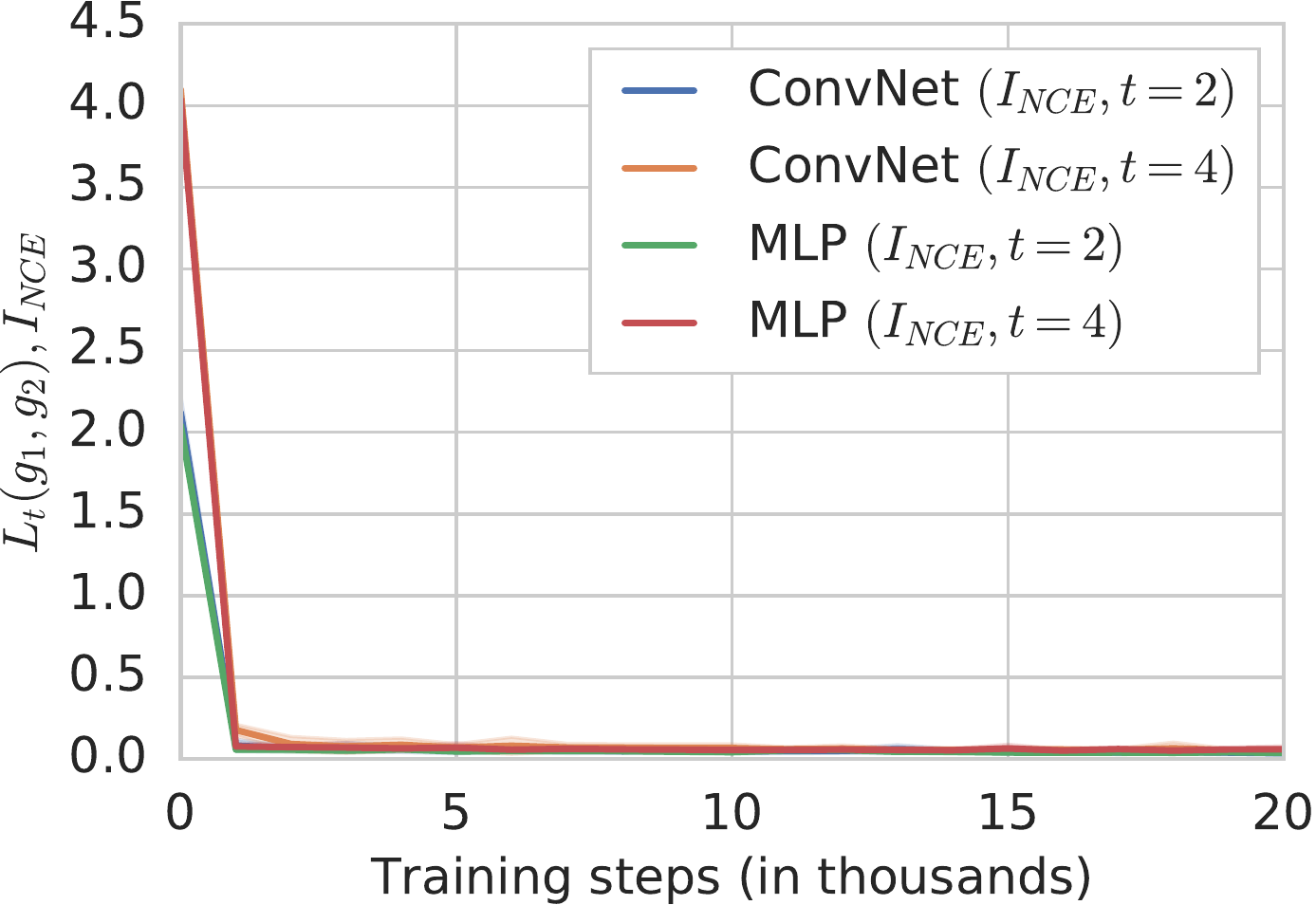}\qquad
    \includegraphics[width=0.35\textwidth]{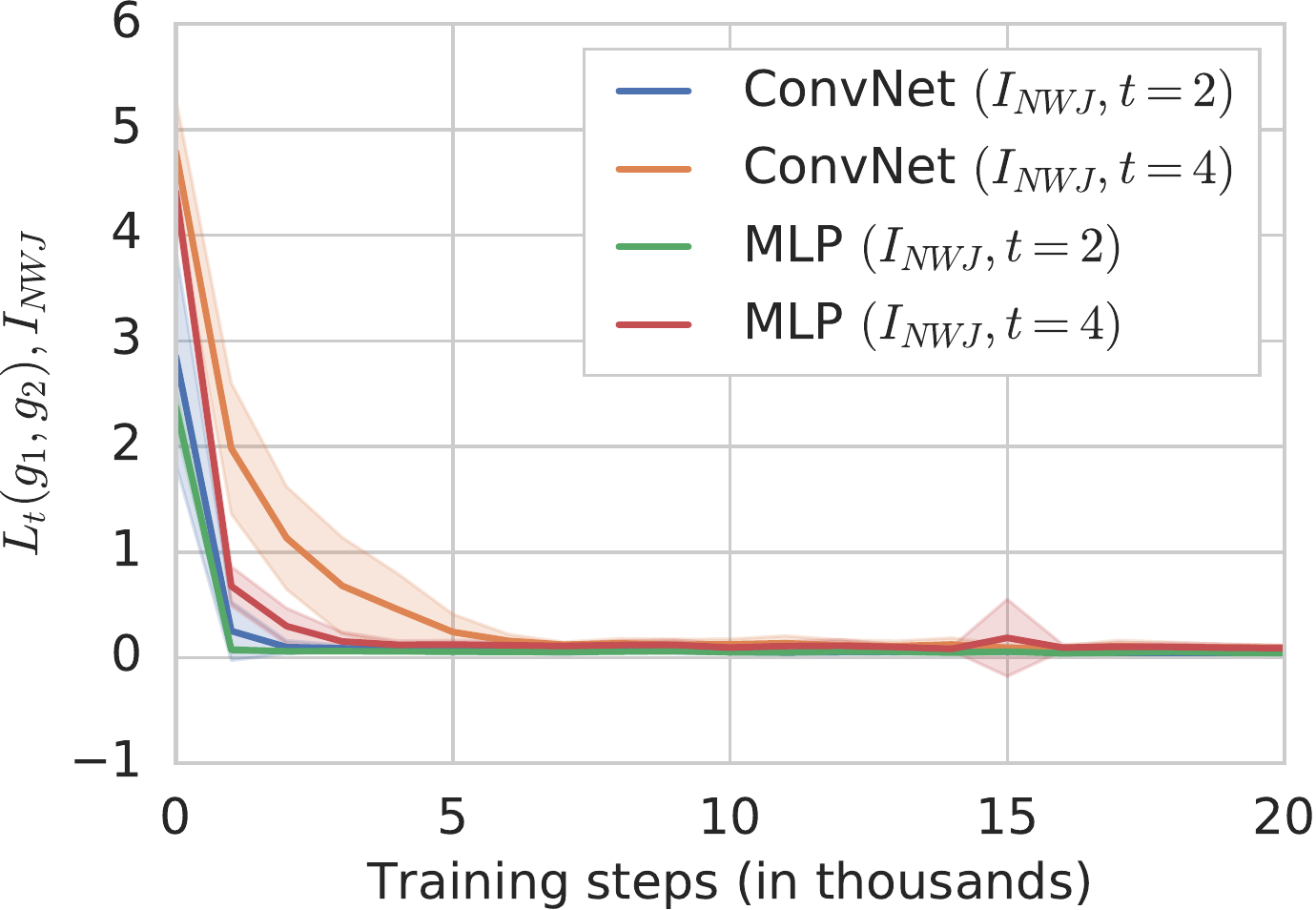}
    \caption{Additional plots Section~\ref{sec:encoderarch}: Testing loss for different encoder architectures and MI estimators, using a bilinear critic trained to match a given target $\iest$ of $t$ (we minimize $L_t(g_1, g_2) = |\iest(g_1(X^{(1)}); g_1(X^{(2)})) - t|$).}
  \end{figure}

\section{Results for the experiments from Sec.~\ref{sec:criticarch} and~\ref{sec:encoderarch} on CIFAR10}\label{app:cifar10results}

We run the experiments form Sections~\ref{sec:criticarch} and~\ref{sec:encoderarch} on CIFAR10 with minimal changes. Specifically, we use the same encoder and critic architectures with the only difference that the input layers of the encoders are adapted to process the (flattened) $32 \times 14 \times 3$ pixel image halves. Furthermore, we reduce the learning rate from $10^{-4}$ to $10^{-5}$ and triple the number of training iterations. Linear classification in pixel space from the upper image halves achieves a testing accuracy of about $24\%$. 

The CIFAR10 results for the experiment investigating the critic architecture (Section~\ref{sec:criticarch}) can be found in Figure~\ref{fig:critic_arch_cifar} and the results for the experiments investigating the encoder architecture (Section~\ref{sec:encoderarch}) in Figure~\ref{fig:encoder_arch_cifar}. The qualitative behavior of the different encoder and critic architectures in terms of downstream testing accuracy and testing $\iest$ is very similar to the one observed for MNIST. The conclusions made for MNIST hence carry over to CIFAR10.

\begin{figure}[t!]
    \centering
    \includegraphics[width=0.4\textwidth]{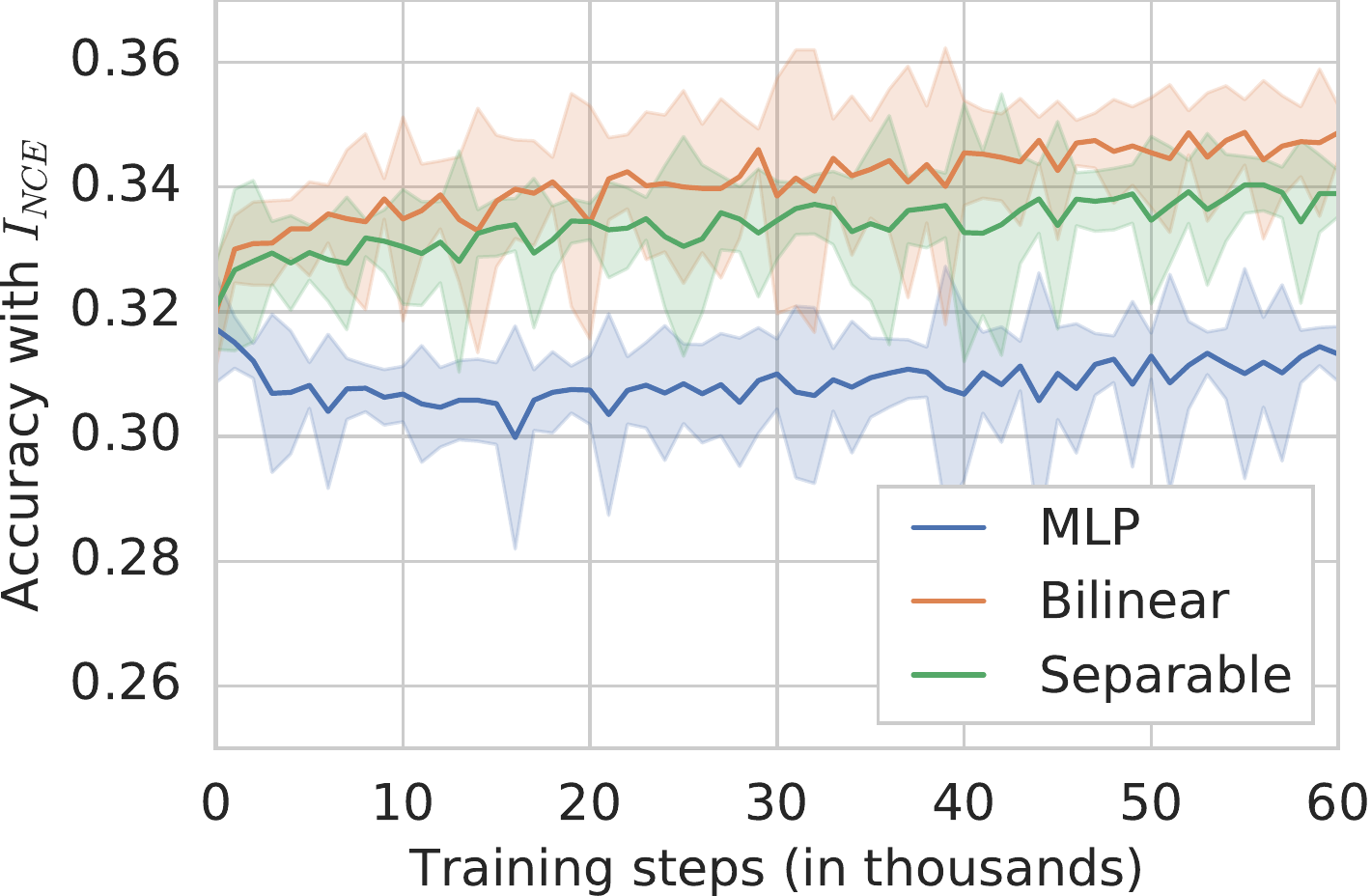} \qquad
    \includegraphics[width=0.4\textwidth]{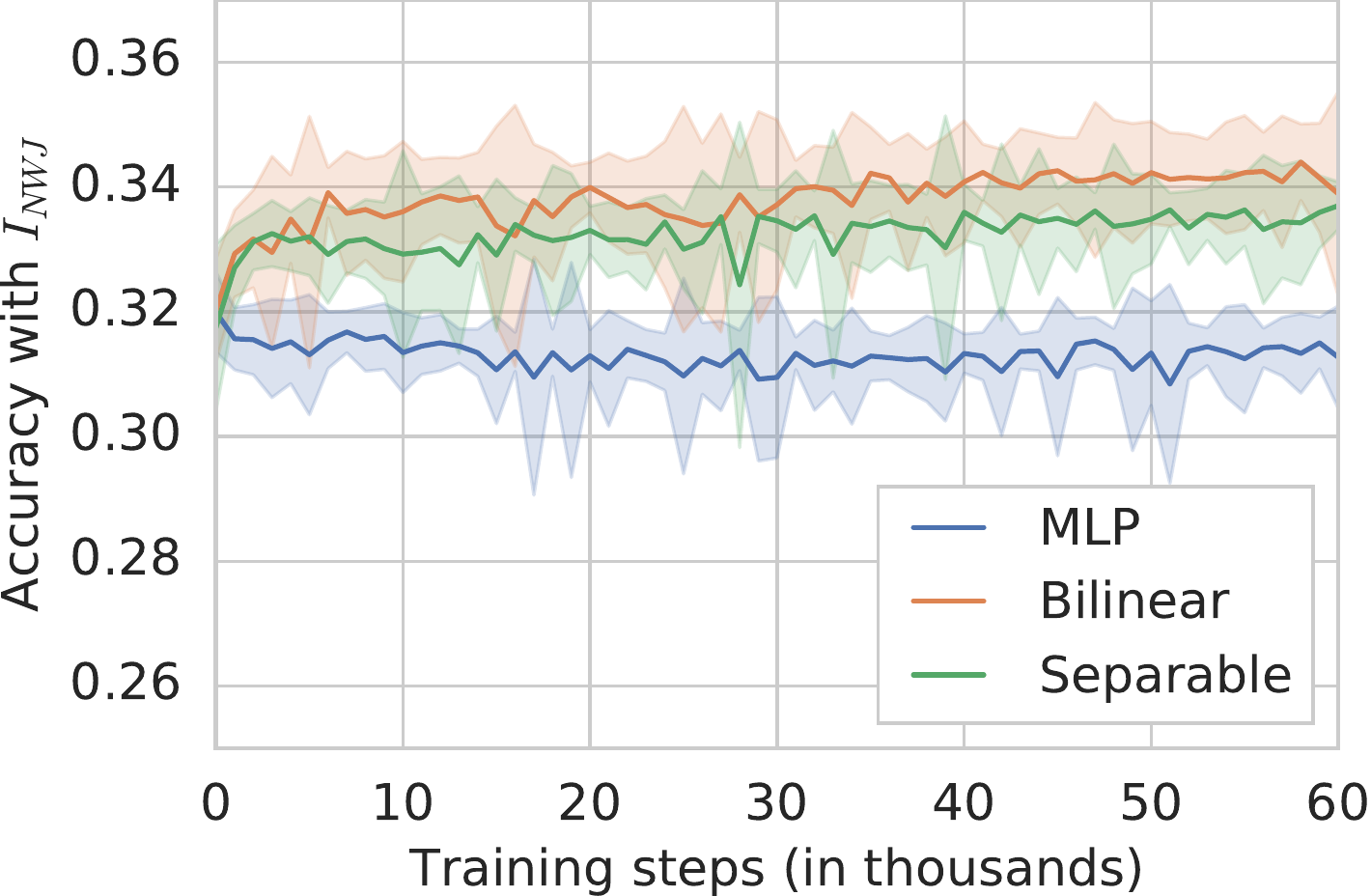} \\[5mm]
    \quad
    \includegraphics[width=0.385\textwidth]{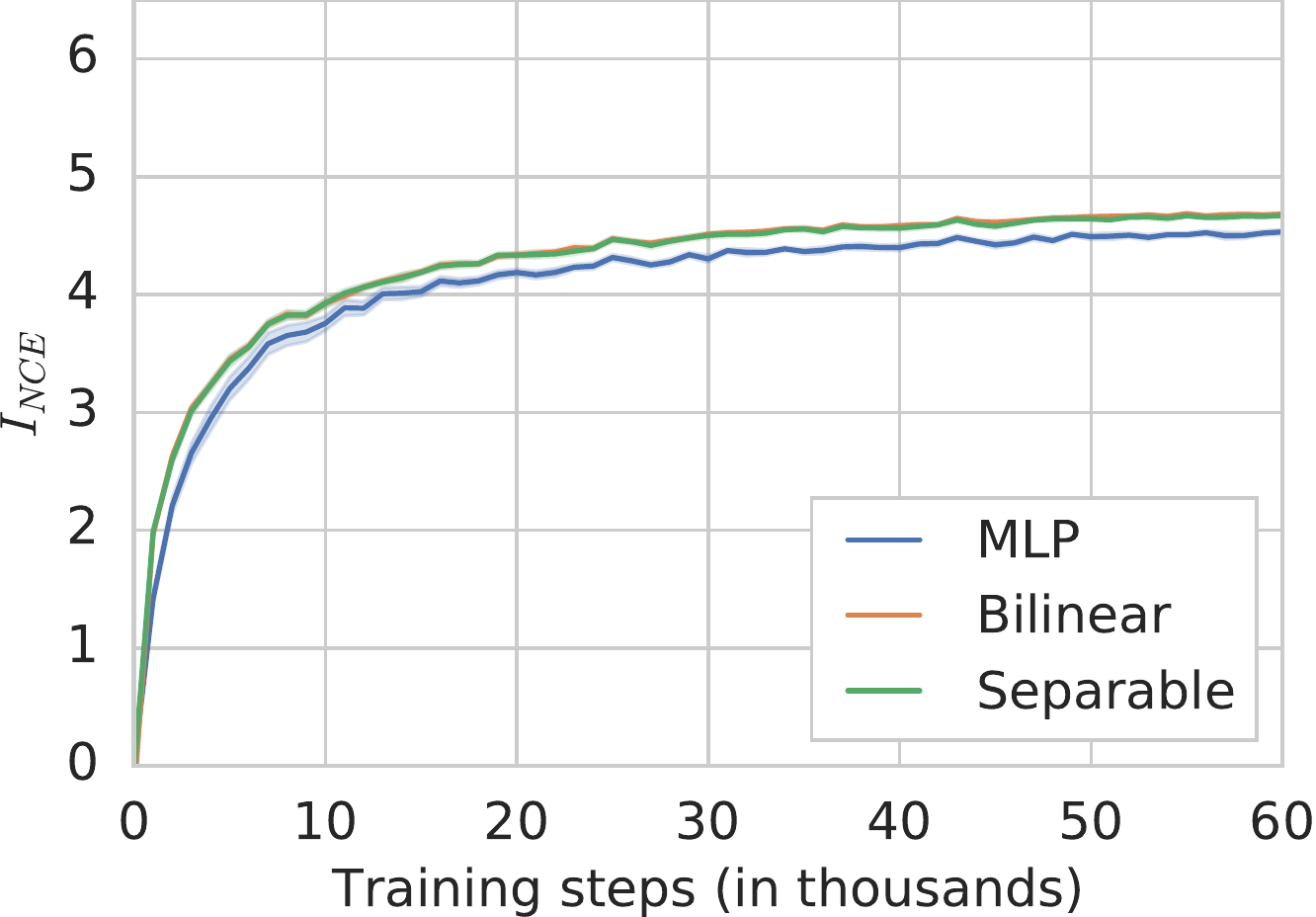} \qquad \quad
    \includegraphics[width=0.38\textwidth]{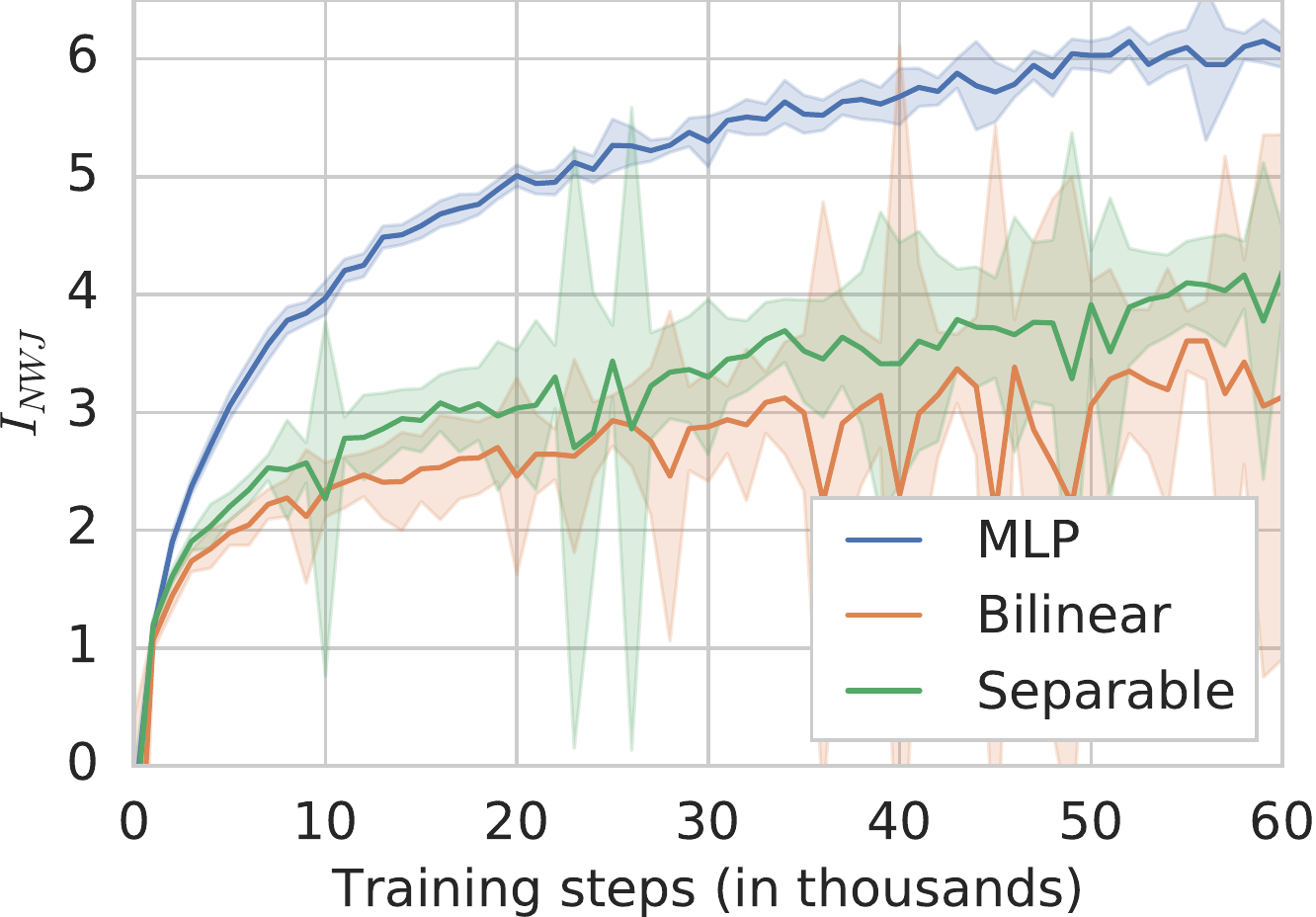}
    \caption{\looseness-1 Downstream testing accuracy for $\infonce$ and $\inwj$ (top row), and corresponding testing $\iest$ value (bottom row) for MLP encoders $g_1, g_1$ and different critic architectures. Bilinear and separable critics lead to higher downstream accuracy than MLP critics, while reaching lower $\inwj$. Note that $\inwj$ exhibits high variance (which is a known property of $\inwj$ \citep{pmlr-v97-poole19a}).}\label{fig:critic_arch_cifar}
\end{figure}

\begin{figure}[t!]
  \centering
  \begin{subfigure}{0.4\textwidth}
    \includegraphics[width=\textwidth]{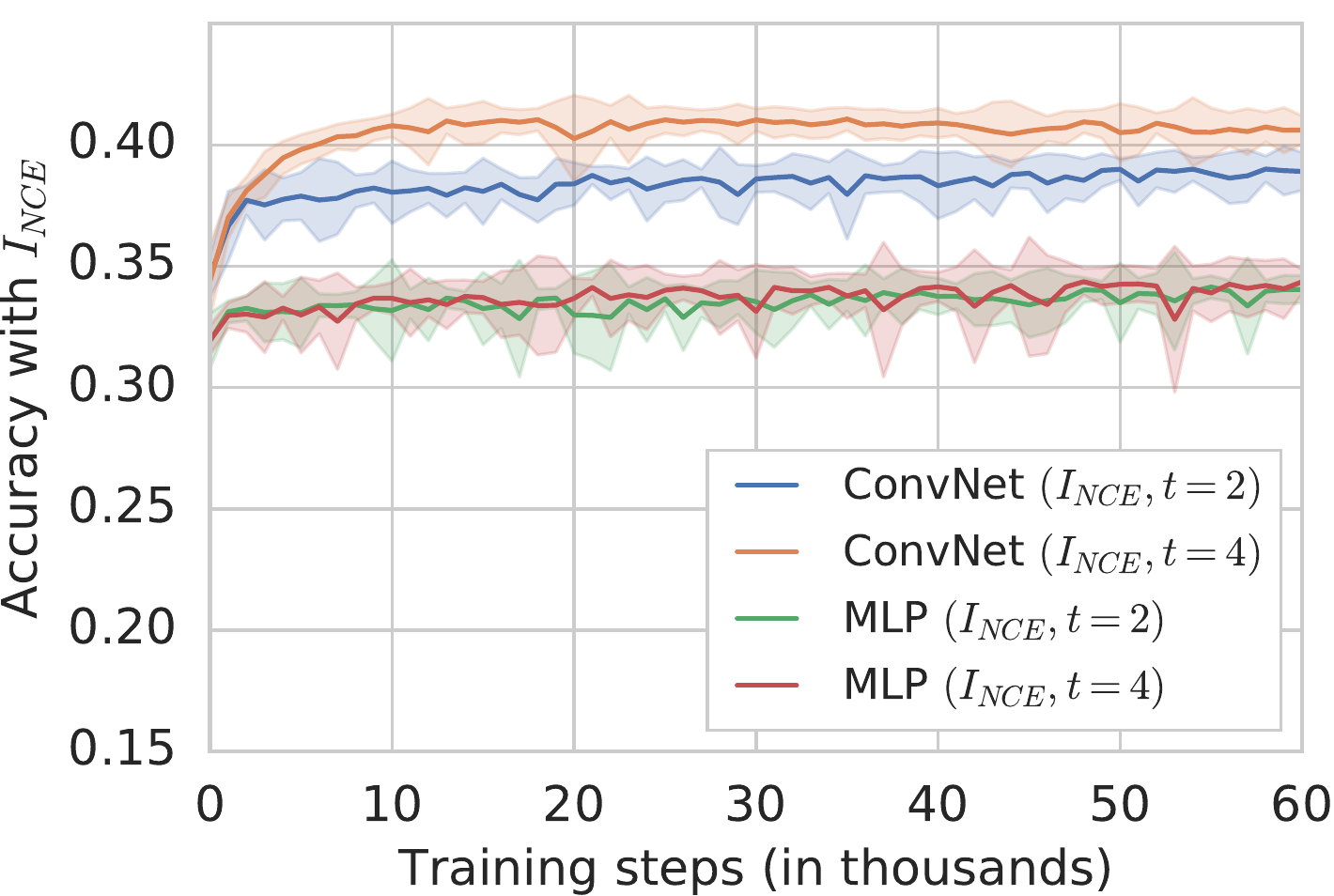}
  \end{subfigure} \qquad
  \begin{subfigure}{0.4\textwidth}
    \includegraphics[width=\textwidth]{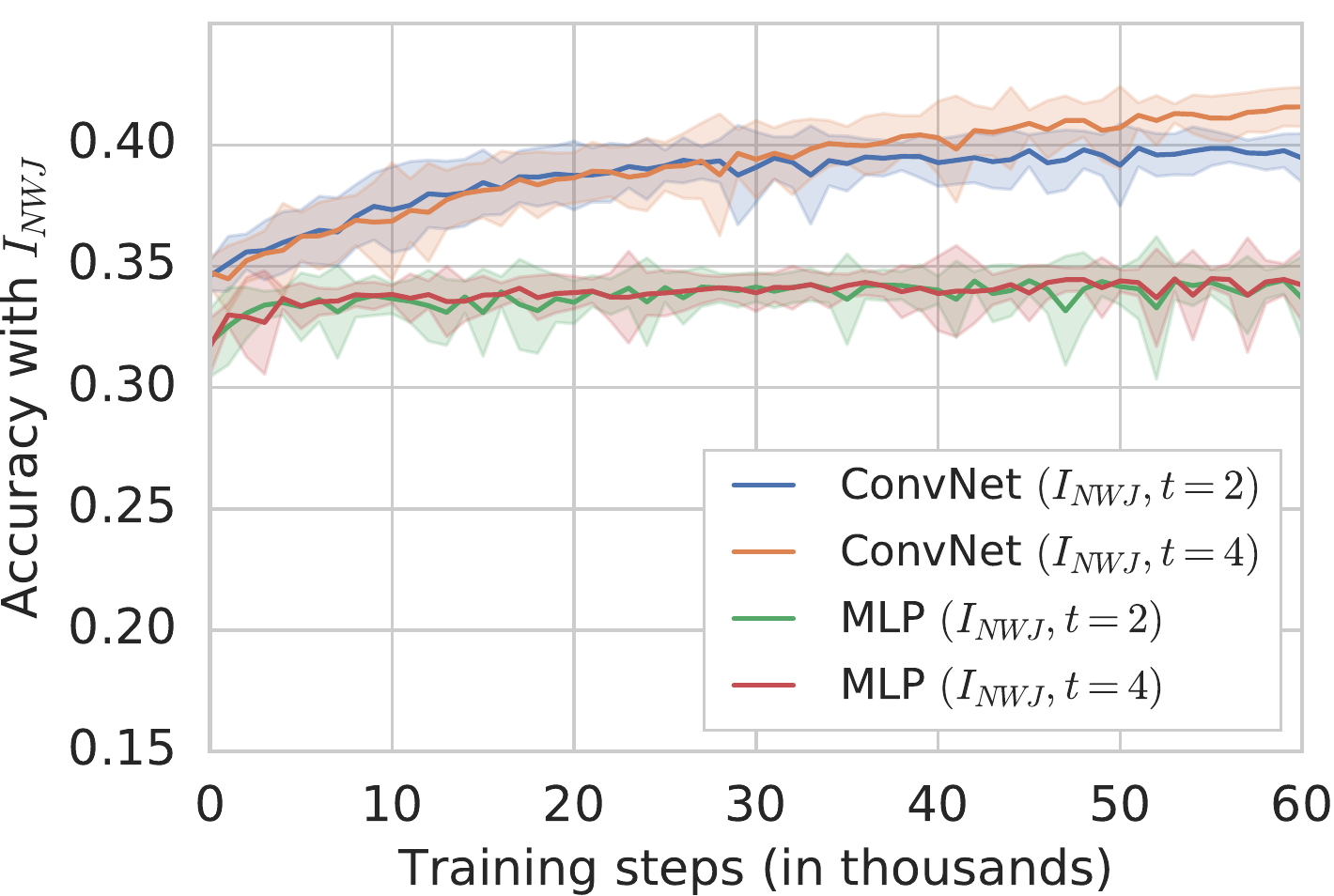}
  \end{subfigure} \\[5mm]
  \quad
  \begin{subfigure}{0.38\textwidth}
    \includegraphics[width=\textwidth]{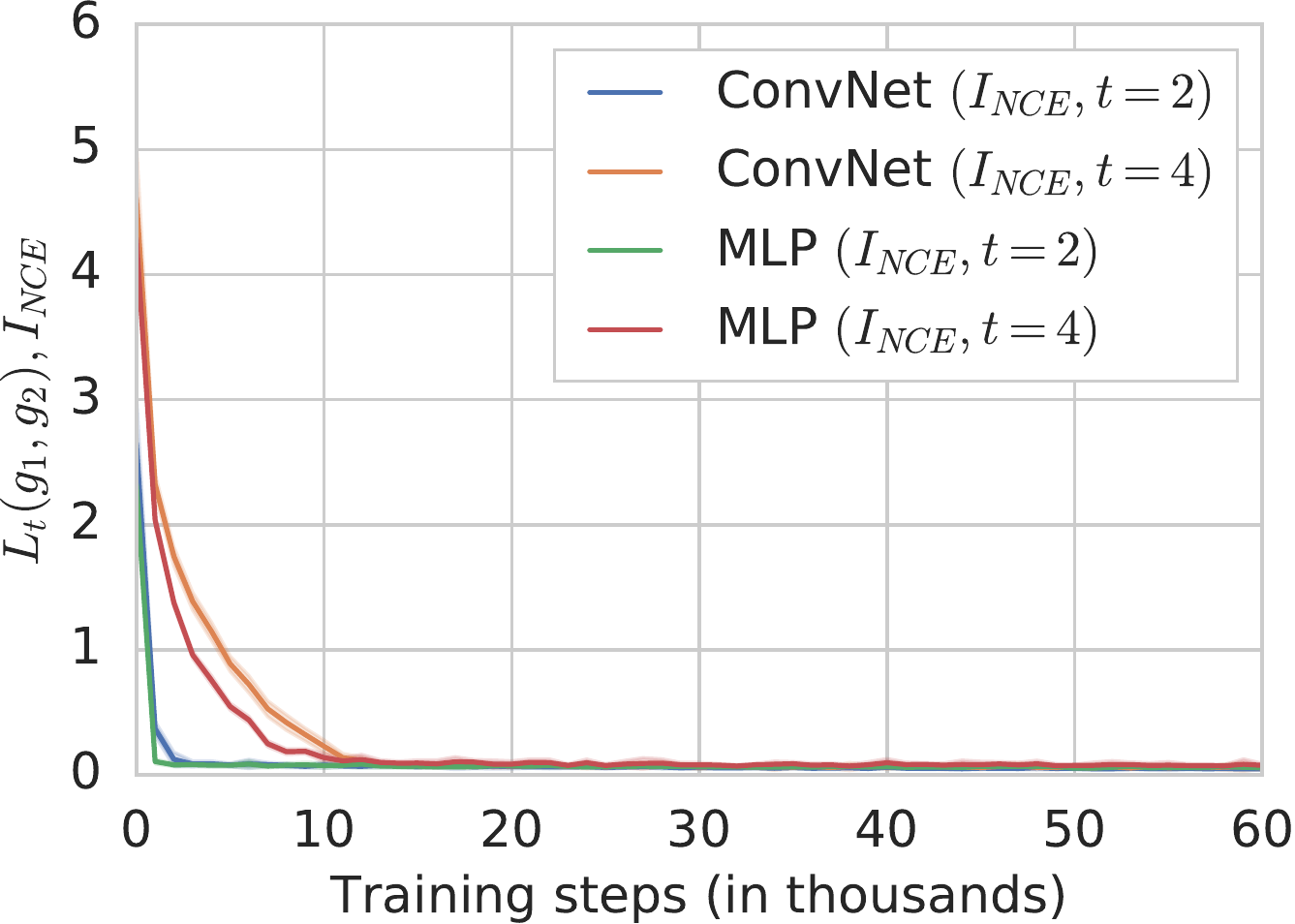}
  \end{subfigure} \qquad \quad
  \begin{subfigure}{0.38\textwidth}
    \includegraphics[width=\textwidth]{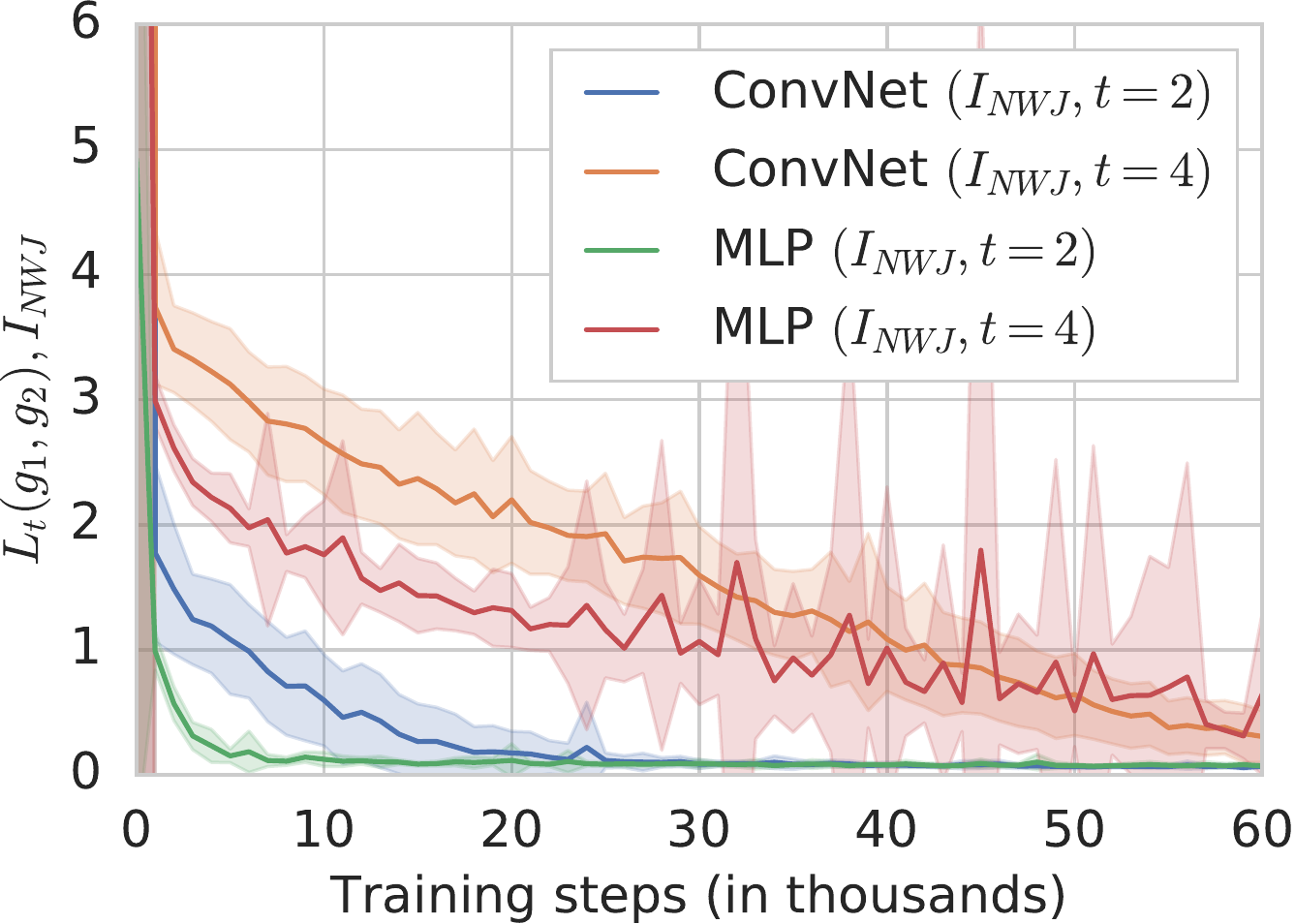}
  \end{subfigure}
  \caption{Downstream testing accuracy (top row) and testing loss value (bottom row) for different encoder architectures and MI estimators, using a bilinear critic trained to match a given target $\iest$ of $t$ (we minimize $L_t(g_1, g_2) = |\iest(g_1(X^{(1)}); g_1(X^{(2)})) - t|$). For a given estimator and $t$, ConvNet encoders clearly outperform MLP encoders in terms of downstream testing accuracy.}
  \label{fig:encoder_arch_cifar}
\end{figure}

\end{document}